%% file: UstatisticsApproach_for_ugm.tex
\documentclass[generic]{imsart}

\RequirePackage[OT1]{fontenc}
\RequirePackage{amsthm,amsmath,amsfonts}
\RequirePackage[round]{natbib}
\RequirePackage[colorlinks,citecolor=blue,urlcolor=blue]{hyperref}
\RequirePackage{algorithm}
\RequirePackage{algorithmic}
\RequirePackage{graphicx}
\RequirePackage{rotating}
\RequirePackage{subcaption}
\usepackage{cleveref}
\usepackage{subeqnarray}

\newcounter{termcounter}
\renewcommand{\thetermcounter}{\Alph{termcounter}}
\crefname{term}{term}{terms}
\creflabelformat{term}{#2\textup{(#1)}#3}

\makeatletter
\def\term{\@ifnextchar[\term@optarg\term@noarg}
\def\term@optarg[#1]#2{%
  \textup{#1}%
  \def\@currentlabel{#1}%
  \def\cref@currentlabel{[][2147483647][]#1}%
  \cref@label[term]{#2}}
\def\term@noarg#1{%
  \refstepcounter{termcounter}%
  \textup{\thetermcounter}%
  \cref@label[term]{#1}}
\makeatother

\startlocaldefs
\numberwithin{equation}{section}
\theoremstyle{plain}
\newtheorem{Def}{Definition}
\newtheorem{Thm}{Theorem}
\newtheorem{Lem}{Lemma}
\newtheorem{corollary}{Corollary}

\newtheorem{remark}{Remark}

\newcommand{\Cov}{\operatorname{Cov}}
\newcommand{\Tr}{\operatorname{Tr}}
\newcommand{\E}{\operatorname{E}}
\newcommand{\Var}{\operatorname{Var}}
\newcommand{\Xbar}[1]{\overline{{#1}}}

\newcommand{\myspace}{\hspace{4pt}}

\def\ci{\perp\!\!\!\perp}
\endlocaldefs
\newcounter{qcounter}
\newcommand\projecturl{\url{https://github.com/wbounliphone/Ustatistics_Approach_For_SD }}

\begin{document}

\begin{frontmatter}
\title{A U-statistic Approach to Hypothesis Testing for
Structure Discovery in Undirected Graphical Models}
\runtitle{Hypothesis Testing for
Structure Discovery}

\begin{aug}
\author{\fnms{Wacha} \snm{Bounliphone}\thanksref{t1}\ead[label=e1]{wacha.bounliphone@centralesupelec.fr}}
\and
\author{\fnms{Matthew B.} \snm{Blaschko}\thanksref{t2}\ead[label=e2]{matthew.blaschko@esat.kuleuven.be}}

\address{Inria Saclay, Galen Team \\
CentraleSup\'{e}lec, L2S \& CVN, Universit\'{e} Paris-Saclay\\
Grande Voie des Vignes\\
92295 Ch\^{a}tenay-Malabry, France\\
}

\address{Center for Processing Speech \& Images\\
Departement Elektrotechniek, KU Leuven\\
Kasteelpark Arenberg 10\\
3001 Leuven, Belgium\\
\printead{e1,e2}\\
}

\thankstext{t1}{WB is supported in part by a CentraleSup\'{e}lec Fellowship and ERC Grant 259112.}
\thankstext{t2}{This work is supported in part by Internal Funds KU Leuven and FP7-MC-CIG 334380. We thank Jonas Peters for helpful discussions and comments on an early draft of this work.}
\runauthor{Bounliphone \& Blaschko}

\affiliation{CentraleSup\'{e}lec, Inria, \& KU Leuven}

\end{aug}

\input{Abstract.tex}

\tableofcontents
\end{frontmatter}

\setlength{\parindent}{0em}
\setlength{\parskip}{1em}

\input{Introduction.tex}

\input{Background.tex}

\input{theory_ustat.tex}

\input{Theory_bound.tex}

\input{Experiments.tex}

\input{Conclusion.tex}

\appendix
\input{Appendix_Proofs_derivation_CovCov.tex}

\bibliography{bibliography}

\end{document}

%% file: Abstract.tex
\begin{abstract}
Structure discovery in graphical models is the determination of the topology of a graph that encodes conditional independence properties of the joint distribution of all variables in the model.  For some class of probability distributions, an edge between two variables is present if and only if the corresponding entry in the precision matrix is non-zero.  For a finite sample estimate of the precision matrix, entries close to zero may be due to low sample effects, or due to an actual association between variables; these two cases are not readily distinguishable.  
Many related works on this topic consider potentially restrictive distributional or sparsity assumptions that may not apply to a data sample of interest, and direct estimation of the uncertainty of an estimate of the precision matrix for general distributions remains challenging.  Consequently, we make use of results for $U$-statistics and apply them to the covariance matrix.  By probabilistically bounding the distortion of the covariance matrix, we can apply Weyl's theorem to bound the distortion of the precision matrix, yielding a conservative, but sound test threshold for a much wider class of distributions than considered in previous works.  
The resulting test enables one to answer with statistical significance whether an edge is present in the graph, and convergence results are known for a wide range of distributions. The computational complexities is linear in the sample size enabling the application of the test to large data samples for which computation time becomes a limiting factor.  We experimentally validate the correctness and scalability of the test on multivariate distributions for which the distributional assumptions of competing tests result in underestimates of the false positive ratio.  By contrast, the proposed test remains sound, promising to be a useful tool for hypothesis testing for diverse real-world problems.  Source code for the tests is available for download from \projecturl. 

\end{abstract}

%% file: Introduction.tex
\section{Introduction}

Graphical models are powerful tools for analyzing relationships between a set of random variables, so that key conditional independence properties can be read from a graph.  Learning the structure of an underlying graphical model is of fundamental importance and has applications in a large number of domains - e.g.\ analysis of fMRI brain connectivity, analysis of genes associated with complex human diseases, or analysis of interactions in social networks.  In many contemporary applications, a large, effectively unlimited stream of raw data with unknown multivariate distribution is to be analyzed.  In such scenarios, computation becomes a fundamental limit and methods that can estimate properties of graphical models from very general distributions with computation linear in the number of observations become necessary.  We address this problem setting in this paper by devising a probabilistic bound on the entries of the precision matrix for highly general distributions that decreases in the sample size as $\mathcal{O}(n^{-1/2})$, while maintaining linear time computation.  This bound can then be used to construct a hypothesis test for a graphical model structure, or for upper and lower bounds on the effect between two variates.

We can divide graphical models in two types, namely directed graphical models, e.g.\ Bayesian networks \citep{pearl2014probabilistic,jensen1996introduction, neapolitan2004learning} or undirected graphical models, e.g. Gaussian graphical models \citep{whittaker2009graphical, lauritzen1996graphical,speed1986gaussian}.  Here, we focus on undirected graphical models to exhibit the conditional dependence structure in multivariate distributions.  

Hypothesis testing with statistical measures of dependence is a relatively well developed field with a number of general results. Classical tests such as Spearman's $\rho$~\citep{spearman1904proof}, Kendall's $\tau$~\citep{kendall1938new}, R\'{e}nyi's $\alpha$~\citep{renyi1961measures} and Tsallis' $\alpha$~\citep{tsallis1988possible} are widely applied. Recently, for multivariate non-linear dependencies, novel statistical tests were introduced and some prominent examples include the kernel mutual information \citep{gretton2003kernel}, the generalized variance and kernel canonical correlation analysis~\citep{bach2003kernel}, the Hilbert-Schmidt independence criterion~\citep{GreBouSmoSch05}, the distance based correlation~\citep{szekely2007measuring} and rankings~\citep{heller2012consistent}. Testing the conditional dependence is even more challenging, and only few dependence measures have been generalized to the conditional case~\citep{fukumizu2007kernel,FukumizuBachJordan2009,zhang2012kernel}. We note that their work requires the estimate of a regularization parameter with appropriate asymptotic decrease to estimate the distribution of the test statistic under the null hypothesis, as well as for kernel selection, and has quadratic space usage rendering it inapplicable to very large data sets. Futhermore, \citet{roverato1996standard} provided an asymptotic distribution for the inverse covariance which is Gaussian and this required the computation of the Isserlis matrix of the inverse of the covariance matrix.
These results, however, do not directly extend to the test that we analyze here: that of independence between two variables \emph{conditioned} on all the others:
\begin{equation} \label{EJS:eq:conditional_independence}
X_i \ci X_j | X_{V \setminus \{i,j\}} .
\end{equation}

In the case of multivariate Gaussian distribution, the non-zero entry in the inverse of the covariance matrix can be shown to correspond to the underlying structure of the graphical model \citep{dempster1972covariance}.  This observation has motivated a range of structure discovery techniques in high-dimensional settings, where  $n < p$ (see Table~\ref{EJS:table:notation} for notation).  Estimation of such high-dimensional models has been the focus on recent research~\citep{schafer2005shrinkage,li2006gradient,meinshausen2006,banerjee2008model,Friedman01072008,ravikumar2011high} where methods impose a sparsity constraint on the entries of the inverse covariance matrix.  The consequence of this attractive method to estimate the inverse of the sparse covariance matrix has been the development of diverse statistical hypothesis tests \citep{g2013adaptive,lockhart2014significance,jankova2014confidence}.  Each of these methods explicitly assumes that the data distribution is multivariate Gaussian.  By contrast, we instead focus in this paper on designing a test for the $n > p$ case, and in particular ensure that the test has computational complexity \emph{linear} in $n$, while making minimal distributional assumptions.  These assumptions are: (i) that the covariance matrix exists and an unbiased estimate converges to this matrix (cf.\ Theorem~\ref{EJS:thm:concentration_of_probability}), and (ii) that the eigenvector-eigenvalue product converges at most at the same asymptotic rate as the convergence of the eigenvalues (cf.\ Lemma~\ref{EJS:lemma:precision_matrix_norm_eigenvalues} and \cite{xia2013convergence}).

In the case of non-Gaussian graphical models, several techniques focus on the existence of a relationship between conditional independence and the structure of the inverse covariance matrix. \citet{loh2013structure} have established theoretical results by extending a number of interesting links between covariance matrices and the graphical model in the case of discrete random variables and particularly for tree-structured graphs. 

While there exist many convenient methods using Gaussian multivariate distributions or discrete variables, other distributions pose new challenges in statistical modeling. 
Consequently, we develop a statistically and computationally efficient framework for hypothesis testing of whether an entry of the precision matrix is non-zero based on a data sample from the joint distribution $P_X$. The proposed test not only has asymptotic guarantees, but is sound for all finite sample sizes without the need to set a regularization parameter or perform a computationally expensive bootstrap procedure.

In this paper, we have taken the approach of precisely modeling the joint distribution of the covariance matrix, and using this distribution to probabilistically bound the distortion of the covariance matrix.  The joint distribution of the entries of the covariance matrix is asymptotically Gaussian with known parameters due to the theory of $U$-statistics~\citep{Serfling1981,lehmann1999elements,hoeffding1948class,lee1990u}.  We are then able to make use of Weyl's theorem~\citep{weyl1912asymptotische} to upper bound the distortion of the precision matrix as a function of the distortion of the covariance matrix, which yields an upper bound on the test threshold at a given significance level.  We derive two upper bounds on the test threshold, one of which is strictly tighter than the other, with computational complexities $\mathcal{O}(n p^2 + p^3)$ and $\mathcal{O}(n p^4)$, respectively, where $n$ is the sample size and $p$ is the number of variables.  We also present a simulation study illustrating analytically and experimentally that both of these thresholds are sound for a substantially more general set of distributions compared with competing tests in the literature and decrease as $\mathcal{O}(n^{-1/2})$. 

%% file: Background.tex
\section{Preliminary definitions}

In this section, we give a brief background of undirected graphical models and testing conditional independence 
(section~\ref{EJS:subsec:Graphmodels} and section~\ref{EJS:subsec:TestingCI}) and a basic description of the $U$-statistic estimator for the covariance matrix (section~\ref{EJS:subsec:ustatistic}). 

\begin{table}
\caption{Notation Table}
\label{EJS:table:notation}
\begin{center}
\begin{tabular}{rp{6cm}}
\multicolumn{1}{c}{\bf Notation}  &\multicolumn{1}{c}{\bf Description} \\
\hline \\
$\operatorname{G} = (\operatorname{V},\operatorname{E})$ & Graph $\operatorname{G}$,  where $V$ is a finite set of vertices
 with $|V|=d$, $\operatorname{E} \subseteq \operatorname{V} \times \operatorname{V}$ is a subset
 of ordered pairs of distinct vertices $(i,j)$;\\
$\operatorname{X}$ &  $\operatorname{X}=\{\operatorname{X}_1, ..., \operatorname{X}_p\}$  is a set of random variables of dimension $p$ with sample size $n$;\\
$\Sigma$ & Covariance matrix of $\operatorname{X}$; \\
$\hat{\Sigma}$ & Unbiased estimator of the covariance matrix of $\operatorname{X}$ estimated from $n$ samples;\\
$\Theta$ & Precision matrix of $\Sigma$;\\
$\hat{\Theta}$ & Empirical estimate of the precision matrix;\\
$\Xbar{X}$ and $\Xbar{XY}$ & $\E [X]$ and $\E [XY]$; \\
$(\mathcal{T}_{ij},\hat{\Theta}_{ij},\delta)$ & The statistical test $\mathcal{T}_{ij}$ with statistic $\hat{\Theta}_{ij}$ at a significance level $\delta$;\\
$t$ & The threshold of the test statistic;\\
$U(A)$ & Function returning the upper triangular \\
& part and diagonal of a matrix $A$
\end{tabular}
\end{center}
\end{table}

\subsection{Undirected Graphical Models} \label{EJS:subsec:Graphmodels}

Graphical models blend probability theory and graph theory together. They are powerful tools for analyzing relationships between a large number of random variables \citep{whittaker2009graphical, lauritzen1996graphical,koller2009probabilistic}.   A \textit{graph} is set of vertices $\operatorname{V} = \lbrace 1, ... p \rbrace$ and a set of edges $\operatorname{E}(\operatorname{G}) \subseteq \operatorname{V} \times \operatorname{V}$. We study undirected graphical models (also known as Markov random fields). 

\paragraph{Undirected Graphical model} An \textit{undirected graphical model} is a joint probability distribution, $P_X$, defined on an undirected graph G, where the vertices $\operatorname{V}$ in the graph index a collection of random variables $\operatorname{X} =  \lbrace \operatorname{X}_1, ..., \operatorname{X}_p \rbrace$ and
the edges encode conditional independence relationships among random variables
\begin{equation}
P_X \propto \prod_{c\in \mathcal{C}} \Psi_c(X_c)
\end{equation}
where $\mathcal{C}$ is the set of maximal cliques in the graph and $\{\Psi_c\}_{c\in \mathcal{C}}$ are non-negative potential functions.

\subsection{Testing conditional independence in undirected graphical models}\label{EJS:subsec:TestingCI}

Conditional independence (CI) is an important concept in statistics, artificial intelligence, and related fields \citep{dawid1979conditional}.  A common measure for the testing of independence of two variables conditioned on a third variable is the \textit{partial correlation} $\rho_{XY.Z}$. With the assumption that all variables are multivariate Gaussian, the partial correlation is zero if and only if $X$ is conditionally independent from $Y$ given $Z$
\begin{equation}
H_0: \rho_{XY.Z} = 0 \text{ \hspace{0.5cm} vs \hspace{0.5cm} } H_1: \rho_{XY.Z}\neq 0.
\end{equation}
The distribution of the sample partial correlation was described by Fisher~\citep{fisher1924distribution} and we would reject $H_0$ if the absolute value of the test statistic exceeded the critical value from the Student table evaluated at $\delta/2$.  The computational complexity of the partial correlation is $\mathcal{O}(np^2 + p^3)$ which simplifies to $\mathcal{O}(np^2)$ as $n\geq p$. 
However, as mentioned in~\citet[Chap. 26 $\&$ 27]{kendall1946advanced}, this hypothesis test makes a strong assumption that the data are Gaussian distributed, and in particular that the fourth-order moment is equal to 0.

Furthermore, tests of conditional independence can be made without any assumption of normality in the distribution, using for instance the permutation distribution of $\rho_{XY.Z}$ or bootstrap techniques, but this becomes too computationally expensive in practice when $n$ tends to be large.

\subsection{A U-statistic Estimator of the Cross-Covariance}\label{EJS:subsec:ustatistic}

Most of the materials in this subsection can be found in \citet{hoeffding1948class}, \citet[Chap. 5]{Serfling1981}, \citet[Chap. 6]{lehmann1999elements} and \cite{lee1990u}.  Suppose we have a sample $\operatorname{X} = \{(X_{i_1}, ... X_{i_p})\}_{1\leq i \leq n}$ of size $n$
drawn i.i.d.\ from a distribution $P_X$.  A $U$-statistic concerns an unbiased estimator of a parameter $\theta$ of $P_X$ using $\operatorname{X}$. Suppose there is some function $h(X_1, ..., X_q)$ which is an unbiased estimator of $\theta=\E[h(X_1, ... X_q)]$, $h$ is called a kernel of order $q \leq p$ of the estimator.  When we have a sample $\operatorname{X} = \{X_{i_1}, ... X_{i_q})\}_{1\leq i \leq n}$ of size $n$ larger than $p$, we can then construct a $U$-statistic in the following way.
\begin{Def}(U-statistic)
Given a kernel $h$ of order $q$ and a sample $\operatorname{X} = \{X_{i_1}, ... X_{i_q})\}_{1\leq i \leq n}$ of size $n$ larger than $p$, the corresponding $U$-statistic for estimation of $\theta$ is obtained by the following 
\begin{equation}
\hat{U} := \dfrac{1}{(n)_q} \sum_{i^n_q} h(X_{i_1}, ..., X_{i_q})
\end{equation}
where the summation ranges over $q$ indices drawn without replacement from $(1, ..., n)$ and $(n)_q$ is the Pochhammer symbol $(n)_q := \dfrac{n!}{(n-q)!}$.
\end{Def}
\begin{Def}($U$-statistic estimator of the covariance) \label{EJS:def:estimator_Ustat_of_covariance}  Let $u_i = (X_i,Y_i)^T$ be ordered pairs of samples $1 \leq i \leq p$.  Consider $\Sigma = \Cov(X,Y)$, the covariance functional between $X$ and $Y$ and $h$, the kernel of order 2 for the functional $\Sigma$ such that
\begin{equation}
h(u_1,u_2) = \dfrac{1}{2}(X_1 - X_2)(Y_1 - Y_2).
\label{EJS:eq:_kernel_Ustat_of_covariance}
\end{equation}
The corresponding $U$-statistic estimator of the covariance $\Sigma$ is
\begin{align}
\hat{\Sigma} &= \dfrac{1}{n-1} \sum_{i,j=1}^n (X_i - X_j)(Y_i - Y_j) = \dfrac{1}{n-1} \sum_{i=1}^n (X_i - \bar{X})(Y_i - \bar{Y})
\label{EJS:eq:estimator_Ustat_of_covariance}
\end{align}
where $ \bar{X} = \frac{1}{n}  \sum_{i=1}^n X_i$. $\hat{\Sigma}$ can be computed in linear time.
\end{Def}

%% file: theory_ustat.tex
\section{Structure Discovery in Undirected Graphical Models}

In this section, we will use the $U$-statistic estimator of the covariance matrix to define a hypothesis test for discovering the structure of graphical models.  We show that this estimator can be computed in time linear in the number of samples and study its concentration distribution. 
We will denote the covariance matrix by $\Sigma$ with its unbiased estimator $\hat{\Sigma}$ using Definition~\ref{EJS:def:estimator_Ustat_of_covariance}, and $\Theta = \Sigma^{-1}$ for the precision matrix, with $\hat{\Theta}$ its empirical estimate.

\subsection{Discovery based on a U-statistic estimator}

As the distribution of $\hat{\Theta}$ under the null hypothesis is unknown in general, we focus here on $U$-statistic estimates of $\hat{\Sigma}$ and its asymptotic normal distribution to calculate conservative bounds on the threshold for our hypothesis test.  We therefore develop the full covariance between the elements of $\hat{\Sigma}$, which we denote $\Cov(\hat{\Sigma}) \in \mathbb{R}^{\frac{p(p+1)}{2} \times \frac{p(p+1)}{2}}$.  The size of $\Cov(\hat{\Sigma})$ is due to the symmetry of $\hat{\Sigma}$. 

\begin{Thm}\label{EJS:asymptotic_normal_dist_for_Ustatistics_Cov}(Joint asymptotic normality distribution of the covariance matrix) For all $(i,j,k,l)$ range over each of the $p$ variates in a covariance matrix $\hat{\Sigma}$, if $\Var(\hat{\Sigma}_{ij}) > 0$ and $\Var(\hat{\Sigma}_{kl}) > 0$, then 
\begin{equation}
n^{\frac{1}{2}}
\begin{pmatrix}
\hat{\Sigma}_{ij} - \Sigma_{ij}  \\
\hat{\Sigma}_{kl} - \Sigma_{kl}  \\
\end{pmatrix}
\overset{d}{\longrightarrow} \mathcal{N} 
\begin{pmatrix}
\begin{pmatrix}
0  \\
0  \\
\end{pmatrix}, 
\begin{pmatrix} 
\Var(\hat{\Sigma}_{ij}) & \Cov(\hat{\Sigma}_{ij},\hat{\Sigma}_{kl}) \\ 
\Cov(\hat{\Sigma}_{ij},\hat{\Sigma}_{kl}) & \Var(\hat{\Sigma}_{kl})
\end{pmatrix} 
\end{pmatrix}.
\end{equation}

\end{Thm}

\begin{Thm}\label{EJS:theorem:covariance_covariance_ustatistic}(Covariance of the $U$-statistic for the covariance matrix)

We note respectively $h$ and $g$ the corresponding kernel of order 2 for the two unbiased estimates $\hat{\Sigma}_{ij}$ and  $\hat{\Sigma}_{kl}$, where
\begin{align}
h(u_1,u_2) &= \dfrac{1}{2} \left( X_{i_1} - X_{i_2} \right)  \left( X_{j_1} - X_{j_2} \right) \hspace{.05pt}, \mbox{with } u_r=(X_{i_r},X_{j_r})^T \\
g(v_1,v_2) &= \dfrac{1}{2} \left( X_{k_1} - X_{k_2} \right)  \left( X_{l_1} - X_{l_2} \right) \hspace{.05pt}, \mbox{with } v_r=(X_{k_r},X_{l_r})^T.
\end{align}

The low variance, unbiased estimates of the covariance between two $U$-statistics estimates $\hat{\Sigma}_{ij}$ and $\hat{\Sigma}_{kl}$, where $(i,j,k,l)$ range over each of the $p$ variates in a covariance matrix $\hat{\Sigma}$ is
\begin{align}
\Cov(\hat{\Sigma}) := \Cov(\hat{\Sigma}_{ij},\hat{\Sigma}_{kl})  = \binom{n}{2}^{-1} \left( 2 (n-2) \zeta_1  \right) + \mathcal{O}(n^{-2})
\label{EJS:eq:Covariance_of_Ustat_estimator}
\end{align}
where $\zeta_1 = \Cov \left(  \E_{u_2}[h(u_1,u_2)], \E_{v_2}[g(v_1,v_2)] \right)$.

\end{Thm}
\begin{proof}
Eq.~\eqref{EJS:eq:Covariance_of_Ustat_estimator} is constructed with the definition of Covariance of a $U$-statistic as given by \citet{hoeffding1948class}.
\end{proof}

\begin{Thm}
\label{EJS:thm:derivation_seven_exhaustive_cases}
There are seven exhaustive cases which can be used to estimate Eq.~\eqref{EJS:eq:Covariance_of_Ustat_estimator} for all $1 \leq i,j,k,l \leq p$ through simple variable substitution.  Each of these cases has computation linear in $n$.
\begin{list}{Case \arabic{qcounter}:~}{\usecounter{qcounter}}

\item \label{EJS:derivation_covcov_case:cov1} $i \ne j,k,l$; $j \ne k,l$; $k \ne l$
\begin{align}
\zeta_1 &= \dfrac{1}{4} \biggl\{ 
\Xbar{X_{i} X_{j} X_{k} X_{l}} - \Xbar{X_i} \myspace \Xbar{X_{j} X_{k} X_{l}} -  \Xbar{X_j} \myspace  \Xbar{X_{i}  X_{k} X_{l}} \nonumber \\
&\qquad - \Xbar{X_k} \myspace  \Xbar{X_{i} X_{j}  X_{l}} + \Xbar{X_i} \myspace  \Xbar{X_k} \myspace  \Xbar{ X_{j}  X_{l}} + \Xbar{X_j} \myspace  \Xbar{X_k} \myspace  \Xbar{X_{i}  X_{l}} \nonumber \\
&\qquad - \Xbar{X_{i} X_{j} X_{k}} \myspace \Xbar{X_l} + \Xbar{X_i} \myspace \Xbar{X_l} \myspace \Xbar{ X_{j} X_{k}}   + \Xbar{X_j} \myspace \Xbar{X_l} \myspace \Xbar{X_{i}  X_{k}}    \nonumber \\
&\qquad - \left( \Xbar{X_i X_j} - 2  \myspace \Xbar{X_i} \myspace \Xbar{X_j} \right) \left( \Xbar{X_k X_l} - 2  \myspace\Xbar{X_k} \myspace \Xbar{X_l} \right)
\biggr\} 
\end{align}

\item \label{EJS:derivation_covcov_case:cov2} $i=j$; $j \ne k,l$; $k = l$
\begin{align}
\zeta_1 &= \dfrac{1}{4} \biggl\{ 
\Xbar{ X_{i}^2 X_{k}^2} - 2 \myspace \Xbar{X_i} \myspace \Xbar{X_{i} X_{k}^2}  - 2 \myspace \Xbar{X_{i}^2 X_{k_1}}  \myspace \Xbar{X_k} + 4 \Xbar{X_{i} X_{k}} \myspace \Xbar{X_i}  \myspace \Xbar{X_k} \nonumber \\
&\qquad -  \left( \Xbar{X_{i}^2} - 2 \myspace \Xbar{X_i}^2 \right) \left( \Xbar{X_{k}^2} - 2 \myspace \Xbar{X_k}^2 \right) 
\biggr\} 
\end{align}

\item \label{EJS:derivation_covcov_case:cov3} $i=j$; $j \ne k,l$; $k \ne l$
\begin{align}
\zeta_1 &= \dfrac{1}{4} \biggl\{ 
\Xbar{ X_{i}^2 X_{k} X_{l} } - 2 \myspace \Xbar{X_{i} X_{k} X_{l} } \myspace \Xbar{X_i} - \Xbar{X_{i}^2 X_{l}} \myspace \Xbar{X_k}  \nonumber \\
&\qquad\qquad + 2 \myspace \Xbar{X_{i} X_{l}} \myspace \Xbar{X_i} \myspace \Xbar{X_k}  - \Xbar{X_{i}^2 X_{k_1}} \myspace \Xbar{X_l} + 2 \myspace \Xbar{X_{i} X_{k}} \myspace \Xbar{X_i} \myspace \Xbar{X_l} \nonumber \\
&\qquad - \left( \Xbar{X_{i}^2} - 2 \myspace \Xbar{X_i}^2 \right) \left( \Xbar{X_{k} X_{l} } - 2 \myspace \Xbar{X_k} \myspace \Xbar{X_l}\right)
\biggr\}
\end{align}

\item \label{EJS:derivation_covcov_case:cov4} $i=k$; $j \ne i,k,l$; $k \ne l$
\begin{align}
\zeta_1 &= \dfrac{1}{4} \biggl\{ 
\Xbar{X_{i_1}^2 X_{j_1} X_{l_1}} - \Xbar{X_i} \myspace \Xbar{X_{j_1} X_{i_1} X_{l_1}} - \Xbar{X_{i_1}^2 X_{l_1} } \myspace \Xbar{X_j} \nonumber \\ 
&\qquad\qquad - \Xbar{X_{i_1} X_{j_1} X_{l_1}} \myspace \Xbar{X_i}  + \Xbar{X_i}^2 \myspace \Xbar{X_{j_1} X_{l_1}}  + \Xbar{X_{i_1} X_{l_1}} \myspace \Xbar{X_j} \myspace \Xbar{X_i}  \nonumber \\
&\qquad\qquad - \Xbar{X_{i_1}^2 X_{j_1} } \myspace \Xbar{X_l}  + \Xbar{X_i} \myspace \Xbar{X_{j_1} X_{i_1}} \myspace \Xbar{X_l}  + \Xbar{X_{i_1}^2} \myspace \Xbar{X_j} \myspace  \Xbar{X_l} 
\big] \nonumber \\
&\qquad -  \left( \Xbar{X_{i} X_{j}} - 2 \myspace \Xbar{X_i} \myspace \Xbar{X_j} \right)
\left(  \Xbar{X_{i} X_{l}} - 2 \myspace \Xbar{X_i} \myspace \Xbar{X_l} \right) 
\biggr\}
\end{align}

\item \label{EJS:derivation_covcov_case:cov5} $i=k$; $i \ne j$; $j=l$; 
\begin{align}
\zeta_1 &= \dfrac{1}{4} \biggl\{  \Xbar{X_{i}^2 X_{j}^2} - 2 \Xbar{X_{i} X_{j}^2} \myspace \Xbar{X_i} + \Xbar{X_i}^2 \myspace\Xbar{X_{j}^2} 
 - 2 \Xbar{X_{i}^2 X_{j}} \myspace  \Xbar{X_j} + 2 \Xbar{X_i} \myspace \Xbar{X_j}  \myspace\Xbar{X_{j} X_{i}} + \Xbar{X_{i}^2} \myspace  \Xbar{X_j}^2 \nonumber \\
&\qquad \qquad - \left( \Xbar{X_i X_j} - 2 (\Xbar{X_i} \myspace \Xbar{X_j}) \right)^2   \biggr\}  
\end{align}

\item \label{EJS:derivation_covcov_case:cov6} $i=j=k$; $i \ne l$
\begin{align}
\zeta_1 &= \dfrac{1}{4} \biggl\{  
\Xbar{X_{i}^3 X_{l}} - 3 \myspace \Xbar{X_{i}^2 X_{l}} \myspace \Xbar{X_i}  + 2 \myspace \Xbar{ X_{i} X_{l}} \myspace \Xbar{X_i}^2 - \Xbar{X_{i}^3} \myspace \Xbar{X_l} + 2 \myspace \Xbar{X_{i}^2} \myspace \Xbar{X_i} \myspace \Xbar{X_l}  \nonumber \\
&\qquad - \left( \Xbar{X_i^2} - 2 \myspace  \Xbar{X_i}^2 \right)
\left( \Xbar{X_i X_l} - 2 \myspace  \Xbar{X_i} \myspace \Xbar{X_l} \right)
\biggr\}  
\end{align}

\item \label{EJS:case:cov7} $i=j,k,l$
\begin{align}
\zeta_1 &= \dfrac{1}{4} \biggl\{
\Xbar{X_{i}^4} - 4 \Xbar{X_{i}^3} \myspace \Xbar{X_i} + 4 \Xbar{X_{i}^2} \myspace \Xbar{X_i}^2 
- \left( \Xbar{X_i^2} - 2 \Xbar{X_i}^2 \right)^2
\biggr\} 
\end{align}

\end{list}
\end{Thm}
\begin{proof}
A proof of Theorem~\ref{EJS:thm:derivation_seven_exhaustive_cases} is given in Appendix~\ref{EJS:sec:appendix_proof_covcov}.
\end{proof}
We now have that an estimator of a covariance matrix has asymptotic joint Gaussian distribution of its entries.  This may appear contrary to the fact that a covariance matrix lies in the positive definite cone as a Gaussian distribution has unbounded support.  We show here that a Gaussian distribution does not contradict a positive definite covariance matrix by demonstrating concentration of the probability distribution in the positive definite cone.
\begin{Thm}\label{EJS:thm:concentration_of_probability}(Concentration of probability) Let us assume that $\operatorname{X}$ has finite support $[a,b]$ with probability at least $1-\gamma$ for some distribution dependent $\gamma \geq 0$, then for $n>1$ and all $\delta > 0$, with probability at least $(1-\delta)(1-\gamma)$ for all $P_X$
\begin{equation}\label{EJS:eq:concentration_bound_for_covariance_matrix}
\vert \hat{\Sigma}_{ij} - \Sigma_{ij} \vert \leq (b-a) \sqrt{ \log (\delta/2)/n } \quad \forall i,j.
\end{equation}
\end{Thm}
\begin{proof}
The estimator $\widehat{\Sigma}$ of the covariance matrix $\Sigma$ is a $U$-statistic of order 2, where each term is contained in $[a,b]$.  By using the concentration inequality of Hoeffding for $U$-statistics, we achieve
\begin{equation}
2 \operatorname{exp} \left( - \dfrac{2 (n/2) \varepsilon^2 }{(b-a)^2} \right) =\delta
\end{equation}\label{EJS:eq:Hoeffing_bound_for_covariance_matrix}
and obtain $\varepsilon = (b-a) \sqrt{ \log (\delta/2)/n }$.
\end{proof}
If $(1-\delta)(1-\gamma)$ can approach $1$ arbitrarily closely while the r.h.s. of Eq.~\eqref{EJS:eq:concentration_bound_for_covariance_matrix} goes to zero, this concentration of probability will mean that once a sufficient data sample are observed, the maximum eigenvalue of $\Cov(\hat{\Sigma})$ will be much smaller than the smallest eigenvalue of $\Sigma$, and the distribution will be concentrated in the positive definite cone.  An explicit bound on the concentration in the positive definite cone based on Weyl's theorem is employed in the following section to construct our test threshold.

%% file: Theory_bound.tex
\subsection{Hypothesis Test using a U-statistic estimator for the covariance matrix}\label{EJS:subsec:theory_bound}

We now describe a statistical test for structure discovery in graphical models, based on the $U$-statistic estimator $\hat{\Sigma}$ of the covariance matrix.  Given $\operatorname{X}$ a sample matrix of size $n \times p$ and for all $(i,j) \in \lbrace 1, ..., p \rbrace$, the statistical test $(\mathcal{T}_{ij},\hat{\Theta}_{ij},\delta): \left( \operatorname{X}, i,j\right) \longmapsto \lbrace 0,1\rbrace $, is used to distinguish between the following null hypothesis $H_0(i,j)$ and the two-sided alternative hypothesis $H_1(i,j)$:
\begin{equation}
H_0(i,j): \Theta_{i,j} = 0 \text{ \hspace{0.5cm} vs \hspace{0.5cm} } H_1(i,j): \Theta_{i,j} \neq 0
\label{EJS:eq:formulation_test_statistics}
\end{equation}
at a significance level $\delta$. 
This is achieved by comparing the test statistic, $|\hat{\Theta}_{ij}|$ with a particular threshold $t$:  if the threshold is exceeded, then the test rejects the null hypothesis.  The acceptance region of the test is thus defined as any real number below the threshold.

In the following we will explain in Theorem~\ref{EJS:thm:convervative_threshold} how the threshold is determined and show that it is a conservative bound.  To prove Theorem~\ref{EJS:thm:convervative_threshold}, we make use of Lemmas~\ref{EJS:lemma:norm2_less_normfro} and~\ref{EJS:lemma:precision_matrix_norm_eigenvalues}.

\begin{Lem}\label{EJS:lemma:norm2_less_normfro}
With probability at least $1-\delta$
\begin{equation}
\Vert \Sigma - \hat{\Sigma} \Vert_2 \leq \sqrt{2 \lambda_{\max}} \Phi^{-1} \left( 1 - \delta/2 \right)
\end{equation}
where $\Phi(\cdot)$ is the CDF of a standard normal distribution and $\lambda_{\max}$ is the largest eigenvalue of $\Cov(\hat{\Sigma})$.
\end{Lem}
\begin{proof}
As $\hat{\Sigma}$ is a $U$-statistic, we have that $U(\hat{\Sigma})$, a vector containing its upper diagonal component (including the diagonal), is Gaussian distributed with covariance $\Cov(\hat{\Sigma})$ (cf.~Thm~\ref{EJS:asymptotic_normal_dist_for_Ustatistics_Cov},~\ref{EJS:theorem:covariance_covariance_ustatistic}).  Therefore, with probability at least $1-\delta$.
\begin{equation}\label{EJS:eq:largesteigenvalueBound}
\| U(\Sigma) - U(\hat{\Sigma}) \|_2 \leq \sqrt{\lambda_{\text{max}}} \Phi^{-1}\left(1 - \delta/2 \right)
\end{equation}
and furthermore
\begin{equation}
\| \Sigma - \hat{\Sigma} \|_F \leq \sqrt{2} \| U(\Sigma) - U(\hat{\Sigma}) \|_2
\end{equation}
which combined with the fact that $\| \cdot \|_2 \leq \| \cdot \|_F$ yields the desired result. 
\end{proof}
\begin{corollary}\label{EJS:corollary:TraceBound}
With probability with at least $1-\delta$
\begin{equation}\label{EJS:eq:TraceBound}
\Vert \Sigma - \hat{\Sigma} \Vert_2 \leq \sqrt{2 \Tr[\Cov(\hat{\Sigma})]} \Phi^{-1} \left( 1 - \delta/2 \right)
\end{equation}
\end{corollary}

\begin{Lem}\label{EJS:lemma:precision_matrix_norm_eigenvalues}(Bounding the deviation of the empirical precision matrix as a function of eigenvalues)
Given $\operatorname{X}$ a set of random variables drawn from a distribution for which Eq. \eqref{EJS:eq:frob_norm_fromprecisionmatrix_termB} converges at a rate $\mathcal{O}(n^{-1/2})$ with a precision matrix $\Theta$, and an empirical estimate of the precision matrix $\hat{\Theta}$ corresponding to a covariance matrix $\hat{\Sigma}$ with eigenvalues $\hat{\alpha}_1, \dots , \hat{\alpha}_p$, then with high probability
\begin{align}
\vert \hat{\Theta}_{ij} - \Theta_{ij} \vert  \leq \mu \sqrt{\sum_{k=1}^{p} \left( \dfrac{1}{\alpha_k} - \dfrac{1}{\hat{\alpha}_k}\right)^2} \quad \forall i,j \in \{1,...,p\} 
\end{align} 
for a distribution dependent constant $\mu$.
\end{Lem}
\begin{proof}
We denote respectively $\hat{\Sigma}$ the perturbed matrix of $\Sigma$, with $\alpha_1 \geq ... \geq \alpha_p$ the eigenvalues of $\Sigma$ and $\hat{\alpha}_1 \geq ... \geq \hat{\alpha}_p$ the eigenvalues of an empirical estimate of the true covariance matrix $\hat{\Sigma}$, and $\hat{\Theta}$ the perturbed matrix of $\Theta$.  We then have that $\vert \hat{\Theta}_{ij} - \Theta_{ij} \vert \leq \Vert \hat{\Theta} - \Theta \Vert_F$ for all $i,j \in \{1,...,p\}$.  We will use the property of the singular value decomposition that $\hat{\Sigma} = \hat{V} \hat{A} \hat{V}^T$, where $\hat{V}$ is an $n \times n$ unitary matrix and a diagonal matrix $\hat{A}$ with $\hat{A}_{ii} = \hat{\alpha}_{i}$ is the $i$-th eigenvalue of $\hat{\Sigma}$. Furthermore, we have that $\Sigma^{-1} = \Theta$ and the empirical estimate of $\Theta$ is $\hat{\Theta}$ such that $\hat{\Theta} = \hat{U} \tilde{\Lambda} \hat{U}^T$ where $\hat{U}$ is an $n \times n$ unitary matrix and a diagonal matrix $\hat{\Lambda}$ with $\hat{\Lambda}_{ii} = 1 / \hat{\alpha}_i$.
\begin{align}
\Vert \hat{\Theta} - \Theta \Vert_F^2 
&= \Tr\left[ (\hat{\Theta} - \Theta)(\hat{\Theta} - \Theta) \right] \\
&= \Tr\left[ \hat{\Theta}\hat{\Theta} + \Theta\Theta - 2 \hat{\Theta}\Theta \right]  \\
&= \Tr\left[ \hat{\Lambda}\hat{\Lambda} + \Lambda\Lambda - 2 \hat{U} \hat{\Lambda} \hat{U}^T U \Lambda U^T \right] \\
&= \Tr\left[ \hat{\Lambda}\hat{\Lambda} + \Lambda\Lambda -2\Lambda \hat{\Lambda}\right] + 2\Tr\left[\Lambda \hat{\Lambda} -   \hat{U} \hat{\Lambda} \hat{U}^T U \Lambda U^T \right] \\
&=  \label{EJS:eq:frob_norm_fromprecisionmatrix} \underbrace{\sum_{k=1}^{p} \left( \dfrac{1}{\alpha_k} - \dfrac{1}{\hat{\alpha}_k} \right)^2}_{\term[\ref{EJS:eq:frob_norm_fromprecisionmatrix} A]{EJS:eq:frob_norm_fromprecisionmatrix_termA}} + \underbrace{2 \sum_{k=1}^p \frac{1}{\alpha_k \hat{\alpha}_k} - 2 \Tr \left[ \hat{U} \hat{\Lambda} \hat{U}^T U \Lambda U^T \right]}_{\term[\ref{EJS:eq:frob_norm_fromprecisionmatrix} B]{EJS:eq:frob_norm_fromprecisionmatrix_termB}} \\
& \leq \mu \left(\sum_{k=1}^{p} \left( \dfrac{1}{\alpha_k} - \dfrac{1}{\hat{\alpha}_k} \right)^2 \right) \label{EJS:eq:approximation_withkappa_norm_difference_true_estimate_ofPrec}
\end{align} 
The bound in Eq.~\eqref{EJS:eq:approximation_withkappa_norm_difference_true_estimate_ofPrec} will hold with high probability, e.g.\ when the finite moment conditions of \cite{xia2013convergence} are satisfied, as Eq.~\eqref{EJS:eq:frob_norm_fromprecisionmatrix} is then guaranteed to converge with rate $\mathcal{O}(n^{-1/2})$.
\end{proof}
We have now shown that we can compute a bound on the distortion purely from the eigenvalues of $\hat{\Sigma}$. 

\begin{Thm}(Weyl's Theorem, \citet{weyl1912asymptotische}) \label{EJS:theorem:Weylstheorem}
For two positive definite matrices $\Sigma$ and $\hat{\Sigma}$ with corresponding eigenvalues $\alpha_k$ and $\hat{\alpha}_k$, respectively, if
\begin{align}
\vert \alpha_k - \hat{\alpha}_k \vert &\leq \Vert \hat{\Sigma} - \Sigma \Vert_2  \leq \varepsilon
\end{align}
where $0 < \varepsilon < \alpha_k$ $\forall k \in \{ 1,...,p\}$, then
\begin{align}
 \alpha_k - \varepsilon &\leq  \hat{\alpha}_k \leq  \alpha_k + \varepsilon \quad \forall k \in \{ 1,...,p\} .
\label{EJS:eq:covariance_inequality_from_weyl_theorem}
\end{align}
\end{Thm}

\begin{Thm}(Conservative threshold) \label{EJS:thm:conservative_threshold}
For all $(i,j) \in \lbrace 1, ...,p \rbrace$, the threshold $t$ for testing $H_0: $ \textit{$\Theta_{i,j} = 0$} versus the alternative hypothesis $H_1:$ \textit{$\Theta_{i,j} \neq 0$} is given by $\operatorname{P}$ for a small probability $\delta \in (0,1)$ such that 
\begin{equation}
\operatorname{P} \left( | \hat{\Theta}_{i,j} \mathbf{\vert} > t \vert \Theta_{i,j} = 0 \right) < \delta
\end{equation}
where $t$ is a conservative threshold  
\begin{equation}\label{EJS:eq:conservative_threshold}
t = \mu \sqrt { \sum_{k=1}^{p} \left( \dfrac{- \varepsilon}{\hat{\alpha}_k (\hat{\alpha}_k - \varepsilon)} \right)^2 }
\end{equation}
with $\hat{\alpha}_k$ the $k$-th eigenvalue of the empirical covariance matrix $\hat{\Sigma}$, $\mu$ a distribution dependent constant satisfying the inequality~\eqref{EJS:eq:approximation_withkappa_norm_difference_true_estimate_ofPrec}, and $\varepsilon$ is an error bound such that
\begin{equation} \label{EJS:eq:formulation_thresholds}
\varepsilon_{\operatorname{Eig}} = \sqrt{2 \lambda_{max}} \Phi \left( 1 - \delta/2 \right)  \mbox{, or } \varepsilon_{\operatorname{Trace}} = \sqrt{2 \Tr[\Cov(\hat{\Sigma})] } \Phi \left( 1 - \delta/2 \right) 
\end{equation} 
where $\lambda_{max}$ is the largest eigenvalue of $\Cov(\hat{\Sigma})$ and $\Tr[\Cov(\hat{\Sigma})]$ is the trace of $\Cov(\hat{\Sigma})$.
\label{EJS:thm:convervative_threshold}
\end{Thm}
\begin{proof}
We have shown that we can compute the distortion of $\hat{\Theta}$ purely from the eigenvalues of $\Sigma$ and $\hat{\Sigma}$.  Therefore, we use Weyl's theorem on the covariance matrix to get error bounds for the eigenvalues of $\Sigma$.
Inequality~\eqref{EJS:eq:covariance_inequality_from_weyl_theorem} gives the following bounds for the eigenvalues of the precision matrix $\Theta$
\begin{align}
\left(  \dfrac{1}{\alpha_k} - \dfrac{1}{\hat{\alpha}_k} \right)^2 
\leq \left( \dfrac{ - \varepsilon}{\hat{\alpha}_k (\hat{\alpha}_k - \varepsilon)} \right)^2  \quad \forall k \in \{1,...,p\}
\label{EJS:eq:precision_inequality_from_weyl_theorem}
\end{align}
 Combining Eq.~\eqref{EJS:eq:approximation_withkappa_norm_difference_true_estimate_ofPrec} and~\eqref{EJS:eq:precision_inequality_from_weyl_theorem} gives
\begin{equation}
\Vert \hat{\Theta} - \Theta \Vert_F \leq \mu \sqrt{ \sum_{i=1}^{p} \left(  \dfrac{- \varepsilon}{\hat{\alpha}_k (\hat{\alpha}_k - \varepsilon)} \right)^2}
\end{equation}
and
\begin{equation}
\vert \hat{\Theta}_{ij} - \Theta_{ij} \vert \leq \| \hat{\Theta} - \Theta \|_F .
\end{equation}
\end{proof}
\begin{Thm}
For a fixed computational budget $N$ less than the time required to process all data points, and for sufficiently large $p$, the trace bound decreases at the same asymptotic rate as the eigenvalue bound.
\end{Thm}
\begin{proof}
We note that the bound in Corollary~\ref{EJS:corollary:TraceBound} is strictly larger than that of Lemma~\ref{EJS:lemma:norm2_less_normfro}, but its computation  $C_{\operatorname{Trace}}(n,p) \asymp np^2$ as opposed to $C_{\operatorname{Eig}}(n,p) \asymp np^4$, where $\asymp$ denotes that the function is asymptotically bounded above and below \citep{Temlyakov2011}. 
The number of samples processed is $n_{\operatorname{Trace}}(N,p) \asymp N/p^2$ for the trace test and $n_{\operatorname{Eig}}(N,p) \asymp N/p^4$ for the eigenvalue test.

For a full rank $p^2-\binom{p}{2} \times p^2-\binom{p}{2}$ p.s.d.\ matrix, the trace is $\mathcal{O}(p^2 \lambda_{max})$.  We have when the sample sizes are equal $\varepsilon_{\operatorname{Trace}} \in \mathcal{O}(p \varepsilon_{\operatorname{Eig}})$.
Furthermore, Equation~\eqref{EJS:eq:conservative_threshold} is asymptotically linear in $\varepsilon$ as $\varepsilon$ approaches zero from the right, and $\varepsilon_{\operatorname{Eig}} \in \mathcal{O}(\lambda(p) n^{-1/2})$, where $\lambda(p)$ gives the dependence of $\varepsilon_{\operatorname{Eig}}$ on the dimensionality of the data.  Therefore, at a fixed computational budget the eigenvalue threshold is $\mathcal{O}(\lambda(p) n_{\operatorname{Eig}}(n,p)^{-1/2}) = \mathcal{O}(\lambda(p) (Np^{-4})^{-1/2}) = \mathcal{O}(\lambda(p) N^{-1/2}p^{2})$, while the trace threshold is $\mathcal{O}(\lambda(p) p (n_{\operatorname{Trace}}(n,p))^{-1/2}) = \mathcal{O}(\lambda(p)N^{-1/2}p^2)$
\end{proof}
%
For the statistical test $(\mathcal{T}_{ij},\hat{\Theta}_{ij},\delta)$ (cf.\ Eq.~\eqref{EJS:eq:formulation_test_statistics}), if $| \hat{\Theta}_{ij}| \geq t$, then the test rejects the null hypothesis at a significance level $\delta$.

In the simulation study, we set $\mu=1$, which we have empirically validated to result in a sound test threshold for a wide range of distributions.  As discussed below, for a trace threshold on a matrix with condition number $\kappa = \frac{\lambda_{\max}}{\lambda_{\min}}$, the trace over-estimates Eq.~\eqref{EJS:eq:frob_norm_fromprecisionmatrix_termA} by at least a factor of $1+\frac{\left(p^2-\binom{p}{2}-1\right)\lambda_{\min}}{\lambda_{\max}}$, and the resulting test is therefore valid for distributions for which Eq.~\eqref{EJS:eq:frob_norm_fromprecisionmatrix_termB} is asymptotically at most $\frac{p^2-\binom{p}{2}-1}{\kappa}$ as large as Eq.~\eqref{EJS:eq:frob_norm_fromprecisionmatrix_termA}.

The computation of the statistical test for structure discovery in multivariate graphical models is described in detail in Algorithm~\ref{EJS:algo:hypothesis_testing_threshold}. 
\begin{algorithm}
\caption{Hypothesis Testing Using a $U$-statistic estimator for the precision matrix}
\begin{algorithmic}[1]
\REQUIRE
$\delta$, the significance level of the test; $\mu$, a constant satisfying \eqref{EJS:eq:approximation_withkappa_norm_difference_true_estimate_ofPrec};
$\operatorname{X} = (X_1, ..., X_p)$ the set of random variables of dimension $p$ with sample size $n$.

\ENSURE
\STATE Compute $\hat{\Sigma}$, the unbiased estimator of $\Sigma$ from $\operatorname{X}$ (cf. Def.~\ref{EJS:def:estimator_Ustat_of_covariance}).
\STATE Compute $\hat{\Theta} = \hat{\Sigma}^{-1}$, the estimator of the precision matrix.
\STATE Compute $U([\Cov(\hat{\Sigma}_{ij},\hat{\Sigma}_{kl})])$ the upper triangular of the covariance of $U(\hat{\Sigma})$ where $\left(i,j,k,l \right)$ vary over the set of $p$ variables (cf. Thm.~\ref{EJS:thm:derivation_seven_exhaustive_cases}).
\STATE Compute 
\begin{itemize}
\item $\lambda_{max}$, the largest eigenvalue of $\Cov(\hat{\Sigma})$, or 
\item $\Tr[\Cov(\hat{\Sigma})]$, the trace of $\Cov(\hat{\Sigma})$.
\end{itemize}
\STATE Compute one of the two error bounds $\varepsilon$ (cf.Eq.~\eqref{EJS:eq:formulation_thresholds})
\begin{itemize}
\item $\varepsilon_{\operatorname{Eig}} = \sqrt{2 \lambda_{max}} \Phi^{-1} \left( 1 - \delta/2 \right)$, or
\item $\varepsilon_{\operatorname{Trace}} = \sqrt{2 \Tr[\Cov(\hat{\Sigma})]} \Phi^{-1} \left( 1 - \delta/2 \right)$
\end{itemize}
where $\Phi$ is the CDF of a standard normal distribution. \\
\IF{$\varepsilon$ is greater than the smallest eigenvalue of $\hat{\Sigma}$}
\STATE $t = \infty$
\ELSE
\STATE Compute the conservative threshold for the two error bound,\\ $t = \mu \sqrt { \sum_{k=1}^{p} \left(  \dfrac{- \varepsilon}{\hat{\alpha}_k (\hat{\alpha}_k - \varepsilon)} \right)^2 }$,\\ where $\hat{\alpha}_k$ is the $k$-th eigenvalue of the unbiased estimator $\hat{\Sigma}$. 
\ENDIF
\RETURN t. 

\end{algorithmic}
\label{EJS:algo:hypothesis_testing_threshold}
\end{algorithm}

\begin{remark}
In the case that $\varepsilon$ is larger than the smallest eigenvalue of $\hat{\Theta}$, the test threshold is unbounded and we can never reject the null hypothesis.  In this case, additional data are necessary to decrease $\varepsilon$ in order to have a non-trivial bound.  Theorem~\ref{EJS:thm:concentration_of_probability} guarantees that $\varepsilon$ converges to zero as a function of the sample size at a rate $\mathcal{O}(n^{-1/2})$.
\end{remark}

\begin{Thm}
For a test with computational cost $\Omega(n^s)$ and a threshold that decreases as $\Omega(n^r)$, our test is asymptotically more powerful in the regime $n \gg p$ whenever $\frac{r}{s}> -\frac{1}{2}$.
\end{Thm}
\begin{proof}
Our tests have computation $C_{\operatorname{Trace}}(n) \asymp  C_{\operatorname{Eig}}(n) \asymp n$.
The convergence of our test threshold is $\mathcal{O}(n^{-1/2})$ so for a fixed computational budget $N$, the test threshold is $\mathcal{O}(N^{-1/2})$.  For a test with computational cost $\Omega(n^s)$ and a computational budget $N$, $\mathcal{O}(N^{1/s})$ samples will be processed.  As $n^r$ is decreasing in $n$ for any consistent test, this implies that the test threshold is $\Omega(N^{r/s})$ which is asymptotically larger than $\mathcal{O}(N^{-1/2})$ whenever $\frac{r}{s}> -\frac{1}{2}$.
\end{proof}
\begin{corollary}
Any test that is superlinear must have a threshold that converges faster than $\mathcal{O}(n^{-1/2})$ to be asymptotically more powerful at a fixed computational budget than the tests proposed here.
\end{corollary}

%% file: Experiments.tex
\section{Simulation Studies}

In this section, we demonstrate the soundness and effectiveness of the proposed test which enables one to answer if an edge is significantly present in a graph.  This is demonstrated both in terms of experiments on randomly generated Gaussian graphical models with known analytic precision matrices $\Theta$. In all experiments, we have used a significance upper bound of $\delta<0.05$.

In the simulation, we generated the data $\operatorname{X}$ from multivariate Gaussian or Laplace distributions with known analytic precision matrices $\Theta = \Sigma^{-1}$, such that $\Sigma_{ij} = X_i^T X_j / \left( \Vert X_i \Vert_2 \Vert X_j \Vert_2 \right)$ for all $(i,j) \in \{1,...,p \}$.
\begin{enumerate}
\item Multivariate Gaussian distribution
\begin{equation}
f(\operatorname{X},\Sigma) = \frac{1}{\sqrt{2\pi^p \vert \Sigma \vert}} \operatorname{exp} \left\lbrace -\frac{1}{2} \operatorname{X}^T \Theta \operatorname{X} \right\rbrace.
\end{equation}

\item Multivariate Laplace distribution \citep{gomez1998multivariate}
\begin{equation}
f(\operatorname{X},\Sigma) = \frac{p \Gamma \left( \frac{p}{2} \right) }{\pi^{\frac{p}{2}}  \Gamma \left( 1+\frac{p}{\omega} \right) 2^{1 + \frac{d}{\omega}} \vert \Sigma \vert^{- \frac{1}{2}}} \operatorname{exp} \left\lbrace \frac{1}{2} \left[ \operatorname{X}^T \Theta \operatorname{X} \right] ^\frac{\omega}{2} \right\rbrace.
\end{equation}
For $\omega = 1$, the multivariate Laplace distribution is derived.

\end{enumerate}

In Figure~\ref{EJS:fig:sample_vs_threshold}, we plot the sample size sample size for 101 regularly spaces values of  $n \in [10000, 1010000]$ versus the empirical threshold $t_{\operatorname{Eig}}$ and $t_{\operatorname{Trace}}$ (cf.\ Eq.~\eqref{EJS:eq:conservative_threshold}) of the test for different numbers of variables $p$.  We clearly distinguish that the threshold $t_{\operatorname{Eig}}$ based on the eigenvalue bound in Eq.~\eqref{EJS:eq:largesteigenvalueBound} is less than the threshold $t_{\operatorname{Trace}}$ based on the trace bound in Eq.~\eqref{EJS:eq:TraceBound} as predicted by Corollary~\ref{EJS:corollary:TraceBound}.  Furthermore, we see that there is a dependence on the size of the graph, with the bounds growing with the number of variables $p$.

In Figure~\ref{EJS:fig:different_distribution_boxplot_and_samplesize_showing_correct_comportement_thresh}, we illustrate the inequality of Weyl's Theorem (Thm~\ref{EJS:theorem:Weylstheorem}). We show the boxplots of the accurate values of eigenvalues of $\Theta$ obtained from the simulation study described. As expected, for a known precision matrix $\Theta$, the eigenvalues $1 / \alpha_i, i \in \{1,...p\}$ is bounded by the two error bounds $\varepsilon_{\operatorname{Eig}}$ and $\varepsilon_{\operatorname{Trace}}$. As the sample size $n$ increases, the two bounds become tighter. 

Then, we compare our edge detection test with the eigenvalue threshold and the trace threshold (\textit{edgeTest-eig} and \textit{edgeTest-tr}) to the Fisher test (\textit{FisherTest}) described in Section~\ref{EJS:subsec:TestingCI} for different multivariate distributions. The simulations are repeated 100 times to provide statistical significance results.

In Figure~\ref{EJS:figure:different_distribution_fpr_witherrorbar}, we plot the significance level of the test $\delta$ against the false positive rate, which refers to the probability of falsely rejecting the null hypothesis for $n = 100 000 $ and $p=6$. The diagonal dotted black line indicates that the significance level of different tests is equal to false positive rate.  Curves above the diagonal indicate that the test does not obey the semantics of (a bound on) the false positive probability, while a curve under the diagonal indicates that the proposed test is conservative but sound.  For the Gaussian distribution (Fig.~\ref{EJS:fig:gaussian_fpr_Normalized}), the conditional independence test is well calibrated while the proposed test is sound. However, for the Laplace distribution (Fig.~\ref{EJS:fig:laplace_fpr_Normalized}), the Fisher test is not valid while the proposed test is sound. Therefore, in Fig.~\ref{EJS:fig:gaussian_fpr_allentries_Normalized} and Fig.~\ref{EJS:fig:laplace_fpr_allentries_Normalized}, we plot the probability of detecting an edge for all entries on the precision matrix $\Theta$, i.e.\ when $|\Theta_{ij} - \hat{\Theta}_{ij} | > t$ for all $(i,j) \in \{1,...,p\}$. 

In Figure~\ref{EJS:fig:samplesize_vs_power_test_for_differents_distributions}, we compare the power of the tests by plotting the sample size for 101 regularly spaced values of  $n \in [10000, 1010000]$ against the power of the test. As expected, in Figs.~\ref{EJS:fig:gaussian_power_Normalized} and~\ref{EJS:fig:laplace_power_Normalized}, we show that the power of the test increases as the sample size $n$ is increased. In Figs.~\ref{EJS:fig:gaussian_power_Normalized_seeEffect} and~\ref{EJS:fig:histogram_entries_Prec}, we take into account an effect in the graph in the sense that we want to detect edge only when there is a high correlation between two edges in the graph, i.e.\ when $|\Theta_ij| > 0.5$ for all $(i,j) \in \{1,...p\}$.

\begin{figure}[ht!]
\centering
\setlength{\tabcolsep}{0.2em}
\renewcommand{\arraystretch}{0.6}
\begin{tabular}{cc}
\begin{sideways} $\qquad\qquad\qquad$ $\log(t_{\operatorname{Eig}})$ and $\log(t_{\operatorname{Trace}})$ \end{sideways} &\includegraphics[width=0.75\columnwidth]{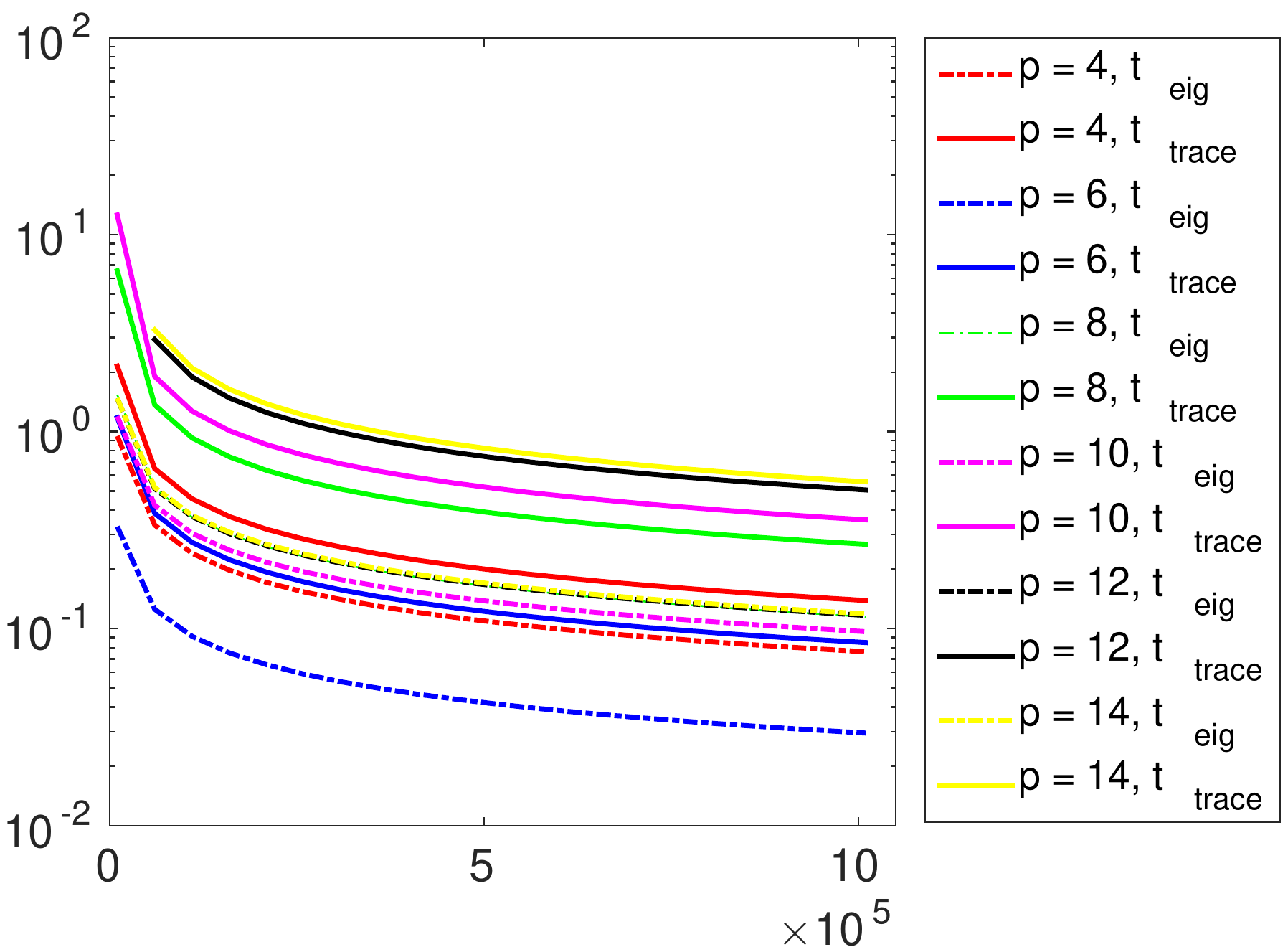} \\
 & Sample size $n$ \\
\end{tabular}
\caption{Illustration of the sample size for 101 regularly spaces values of  $n \in [10000, 1010000]$ versus the thresholds $t_{\operatorname{Eig}}$ and $t_{\operatorname{Trace}}$ (Eq.~\eqref{EJS:eq:conservative_threshold}).  We have plotted both the eigenvalue bound as well as the trace bound (cf.~Lemma~\ref{EJS:lemma:norm2_less_normfro}).  
}
\label{EJS:fig:sample_vs_threshold}
\end{figure}

\begin{figure*}
\centering
\begin{tabular}{cc}
\begin{subfigure}{.45\columnwidth}
\setlength{\tabcolsep}{0.1em}
\renewcommand{\arraystretch}{0.5}
\begin{tabular}{c}
 \includegraphics[width=\columnwidth]{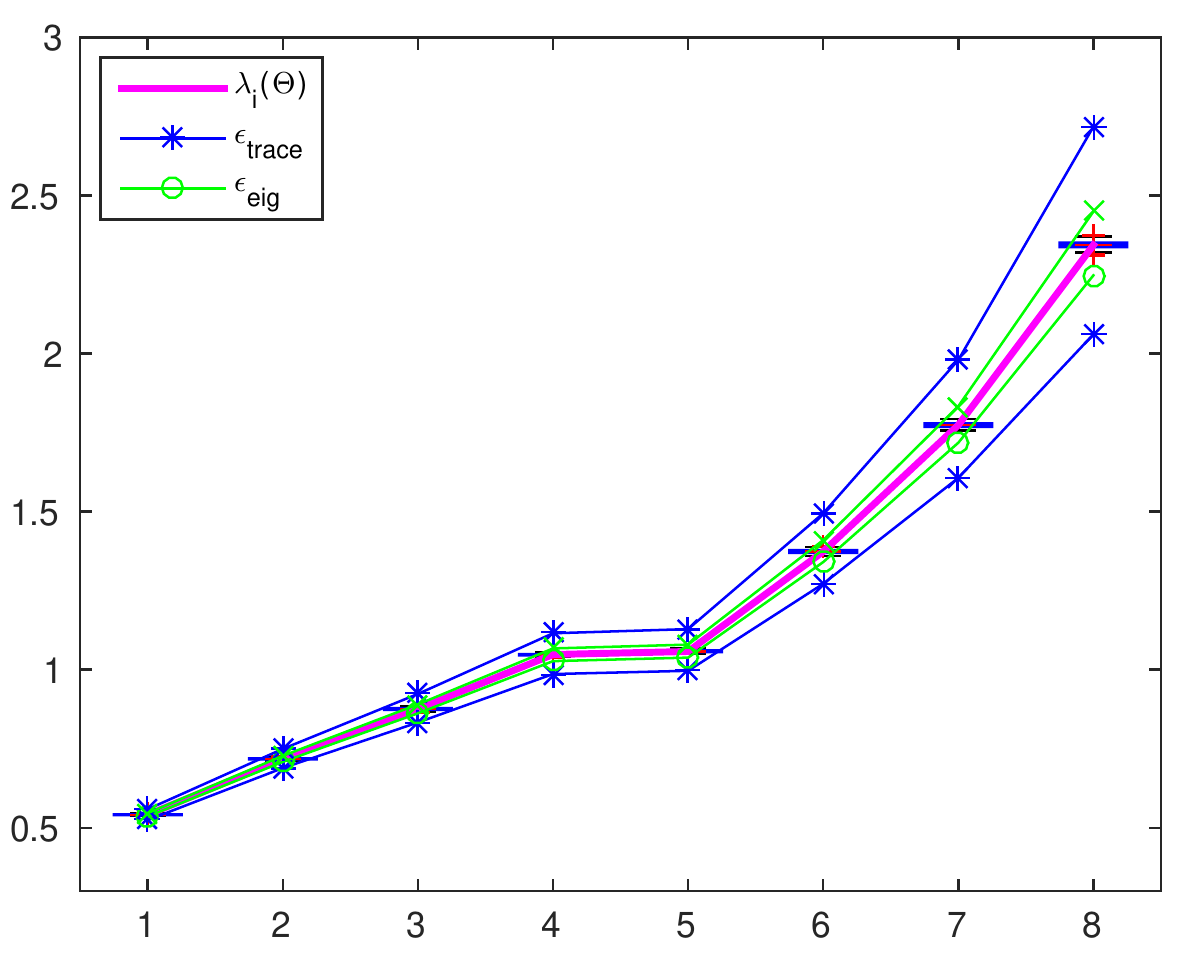} \\
 $1 / \alpha_i(\Theta)$, for $i\in\{1,...,8\}$.
\end{tabular}
\caption{Gaussian dist., $n = 100 000$, $p=6$. }\label{EJS:fig:gaussian_bloxplot_showing_correct_comportement_eigenvalues_m1e5}
\end{subfigure}\hfill
&
\begin{subfigure}{.45\textwidth}
\setlength{\tabcolsep}{0.1em}
\renewcommand{\arraystretch}{0.5}
\begin{tabular}{c}
 \includegraphics[width=\columnwidth]{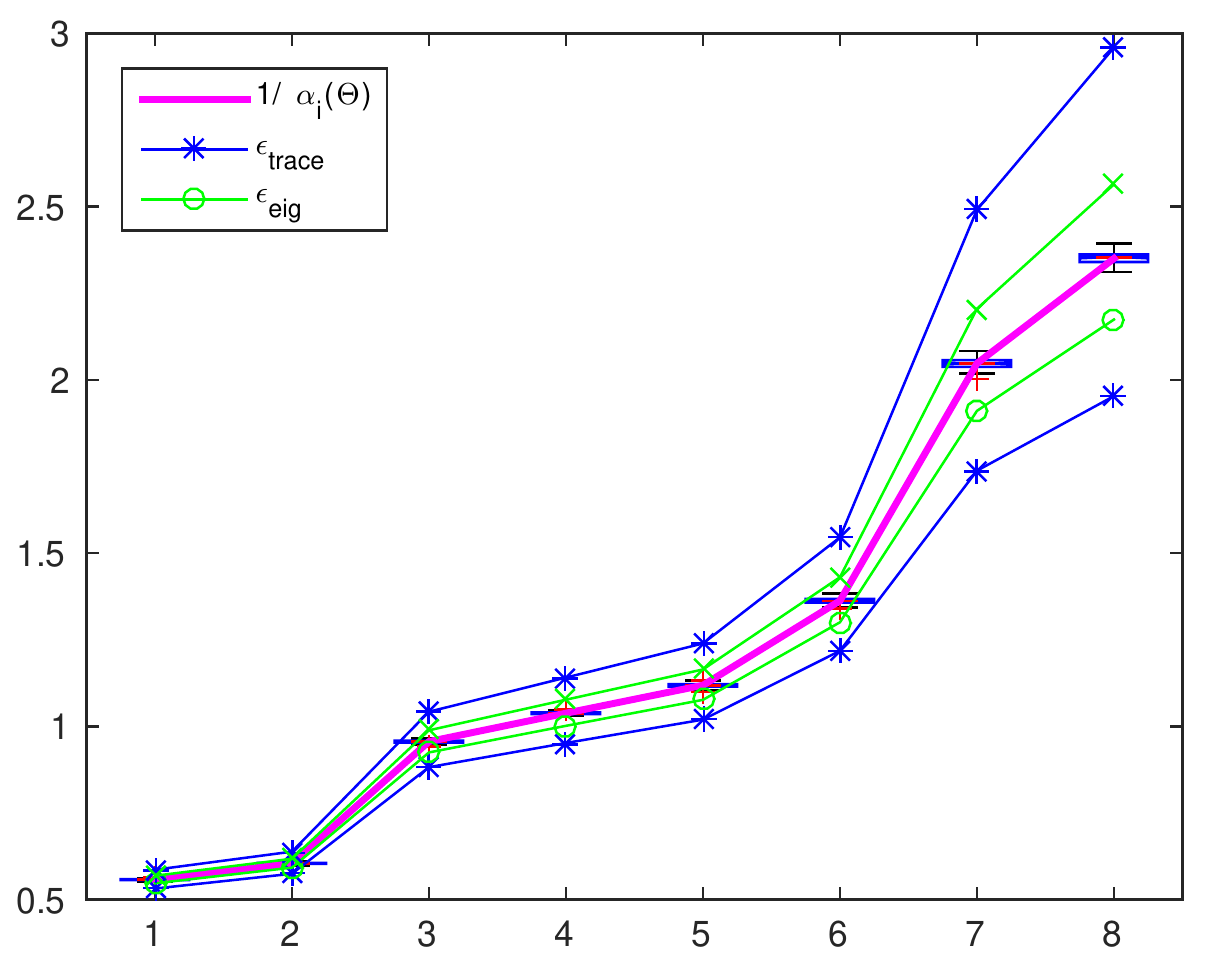} \\
 $1 / \alpha_i(\Theta)$, for $i\in\{1,...,8\}$.
\end{tabular}
\caption{Laplace dist., $n = 100 000$, $p=6$.}\label{EJS:fig:laplace_bloxplot_showing_correct_comportement_eigenvalues_m1e5}
\end{subfigure} 
\\
\begin{subfigure}{.45\columnwidth}
\setlength{\tabcolsep}{0.1em}
\renewcommand{\arraystretch}{0.5}
\begin{tabular}{c}
 \includegraphics[width=\columnwidth]{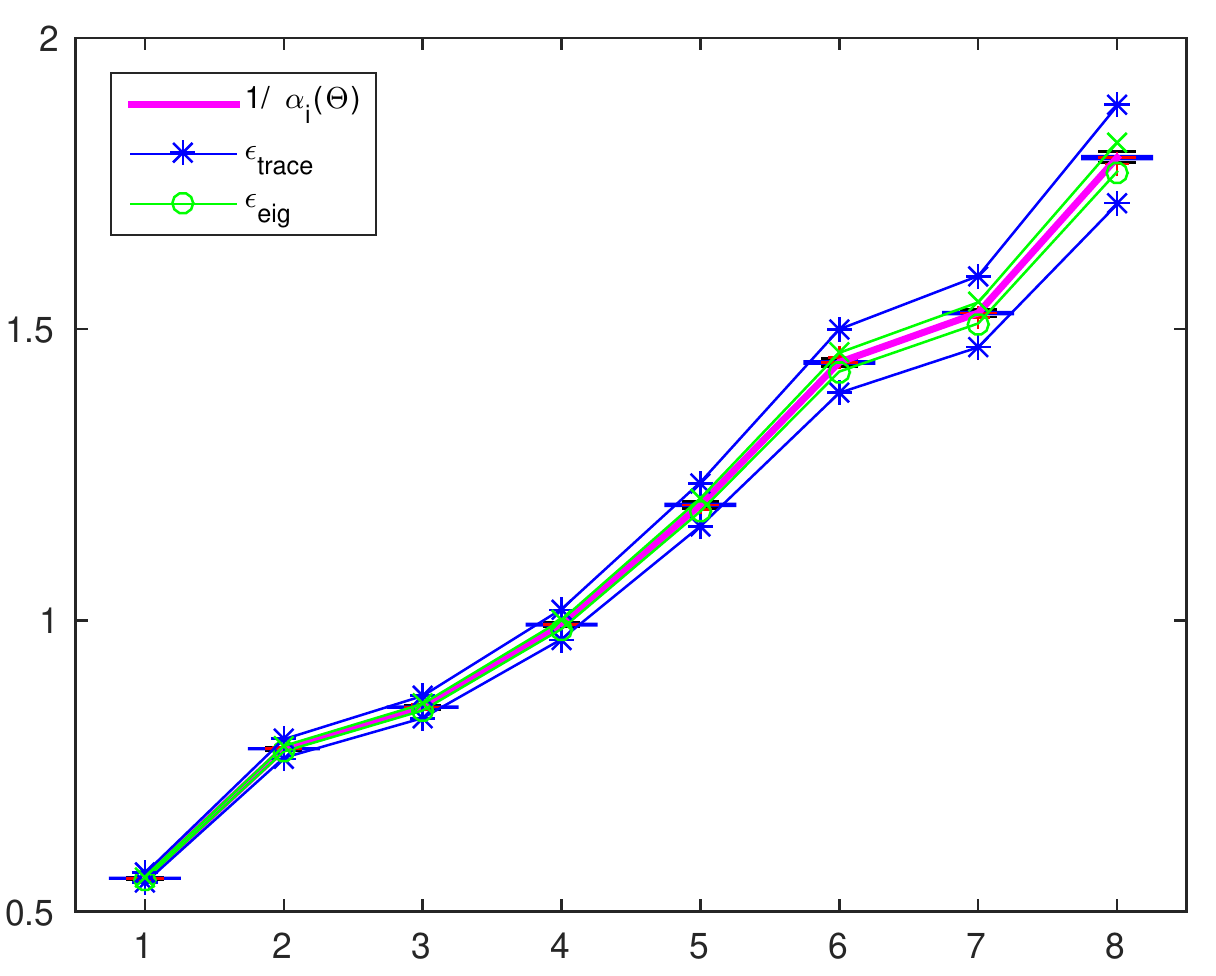} \\
 $1 / \alpha_i(\Theta)$, for $i\in\{1,...,8\}$.
\end{tabular}
\caption{ Gaussian dist, $n = 500 000$, $p=6$.}\label{EJS:fig:gaussian_bloxplot_showing_correct_comportement_eigenvalues_m5e5}
\end{subfigure}\hfill
&
\begin{subfigure}{.45\textwidth}
\setlength{\tabcolsep}{0.1em}
\renewcommand{\arraystretch}{0.5}
\begin{tabular}{c}
 \includegraphics[width=\columnwidth]{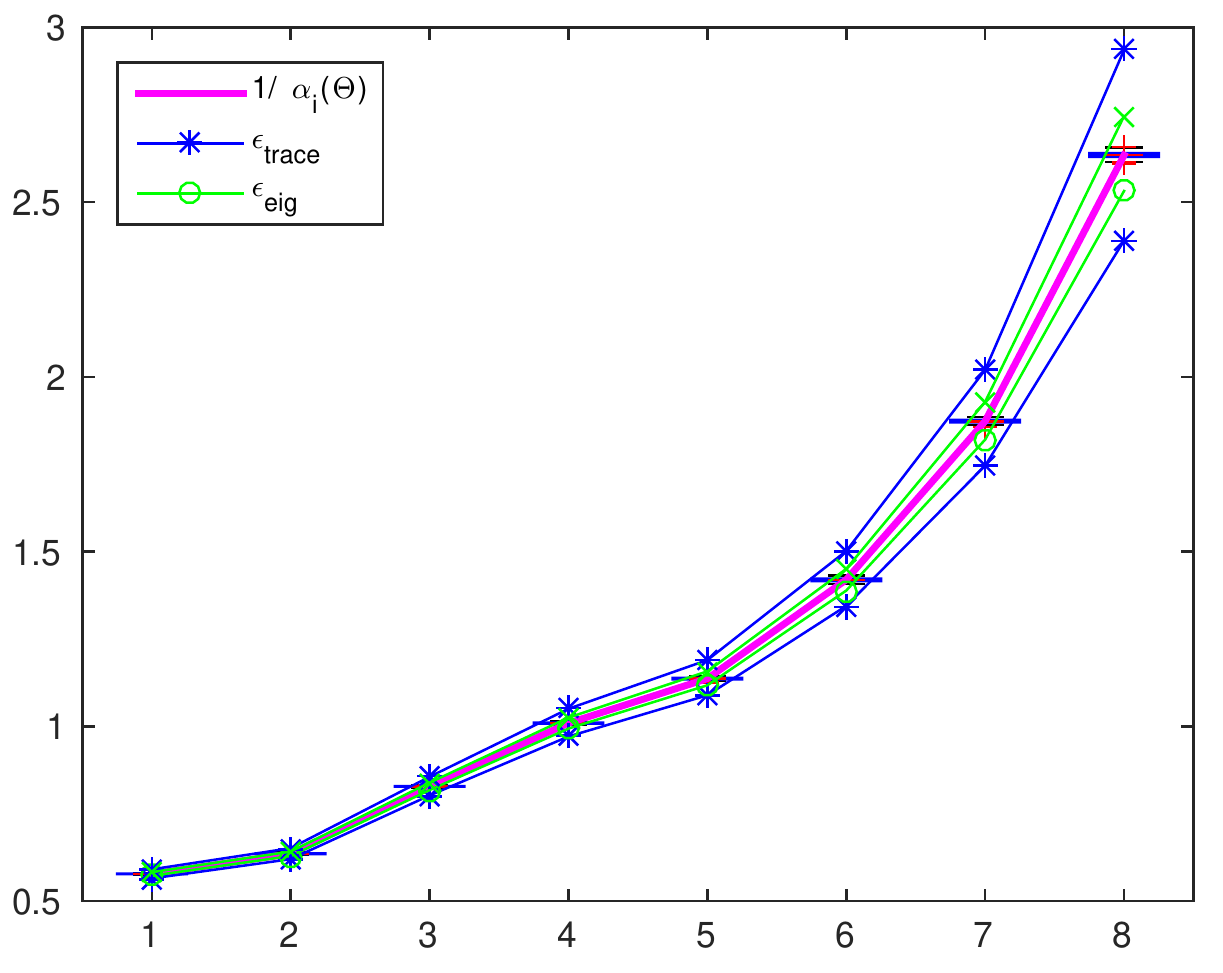} \\
 $1 / \alpha_i(\Theta)$, for $i\in\{1,...,8\}$.
\end{tabular}
\caption{ Laplace dist., $n = 500 000$, $p=6$.}\label{EJS:fig:laplace_bloxplot_showing_correct_comportement_eigenvalues_m5e5}
\end{subfigure}\par
\end{tabular}
\caption{For a known analytic precision matrix $\Theta$ of size $p=8$ and for two different sample sizes,  we show the boxplots of accuracy values of eigenvalues of 200 estimates matrices $\hat{\Theta}$ for the Gaussian (Figs~\ref{EJS:fig:gaussian_bloxplot_showing_correct_comportement_eigenvalues_m1e5}, \ref{EJS:fig:gaussian_bloxplot_showing_correct_comportement_eigenvalues_m5e5}) and Laplace (Figs~\ref{EJS:fig:laplace_bloxplot_showing_correct_comportement_eigenvalues_m1e5}, \ref{EJS:fig:laplace_bloxplot_showing_correct_comportement_eigenvalues_m5e5}) distributions with normalized data. In pink, we plot the true eigenvalue of $\Theta$ and in green and blue, we plot the upper and lower bound given by Weyl's theorem. As $n$ grows, we see that the bound more closely constrains the true eigenvalues of $\Theta$.}
\label{EJS:fig:different_distribution_boxplot_and_samplesize_showing_correct_comportement_thresh}
\end{figure*}

\begin{figure*}
\centering
\begin{tabular}{ccc}

\begin{subfigure}{.45\columnwidth}
\setlength{\tabcolsep}{0.1em}
\renewcommand{\arraystretch}{0.5}
\begin{tabular}{cc}
\begin{sideways} \qquad \qquad false positive rate \end{sideways} & \includegraphics[width=\linewidth]{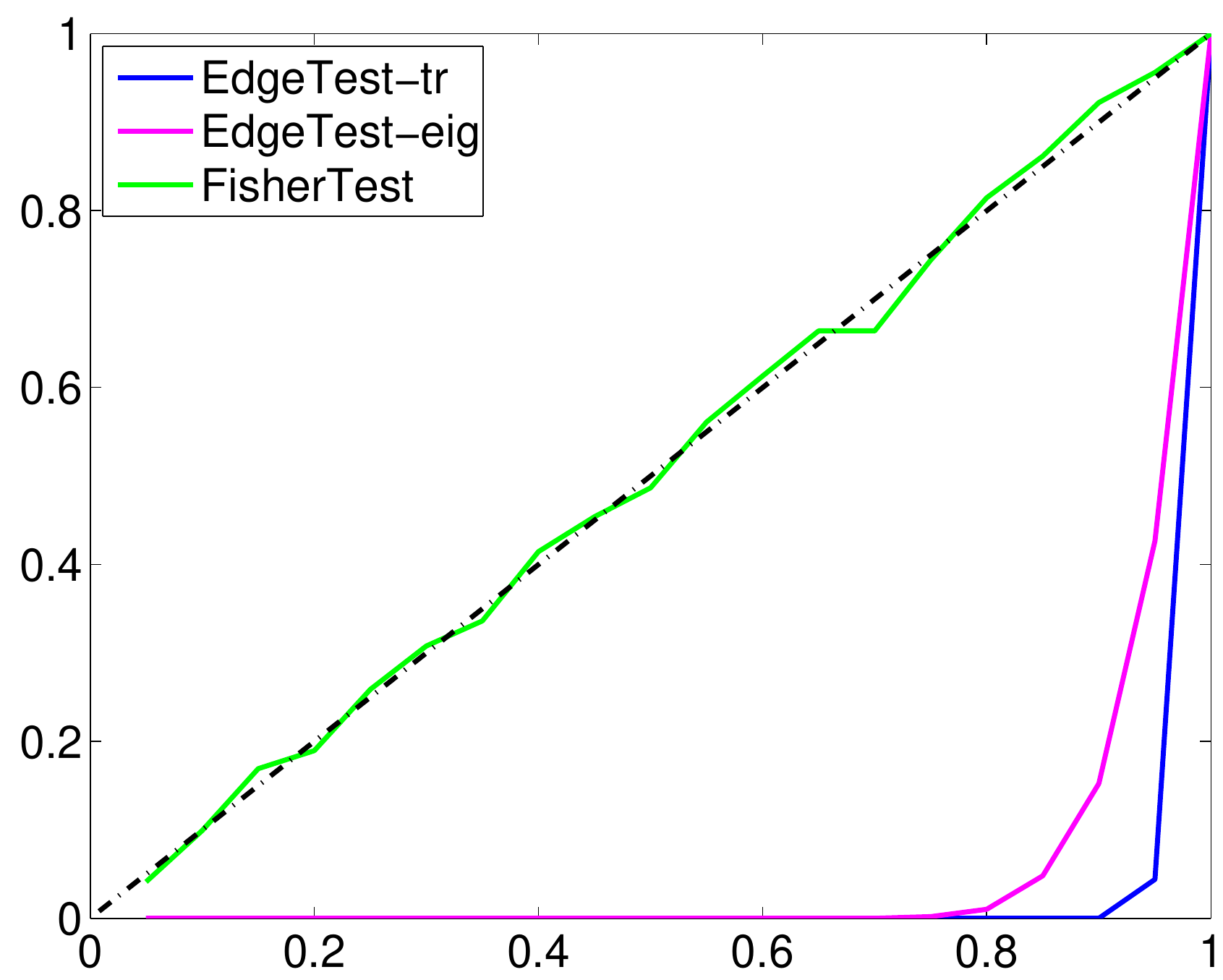} \\
& significance level $\delta$
\end{tabular}
\caption{Gaussian dist., $n = 100 000$, $p=6$.}\label{EJS:fig:gaussian_fpr_Normalized}
\end{subfigure}\hfill

&&

\begin{subfigure}{.45\textwidth}
\setlength{\tabcolsep}{0.1em}
\renewcommand{\arraystretch}{0.5}
\begin{tabular}{cc}
\begin{sideways} \qquad \qquad false positive rate \end{sideways}
& \includegraphics[width=\columnwidth]{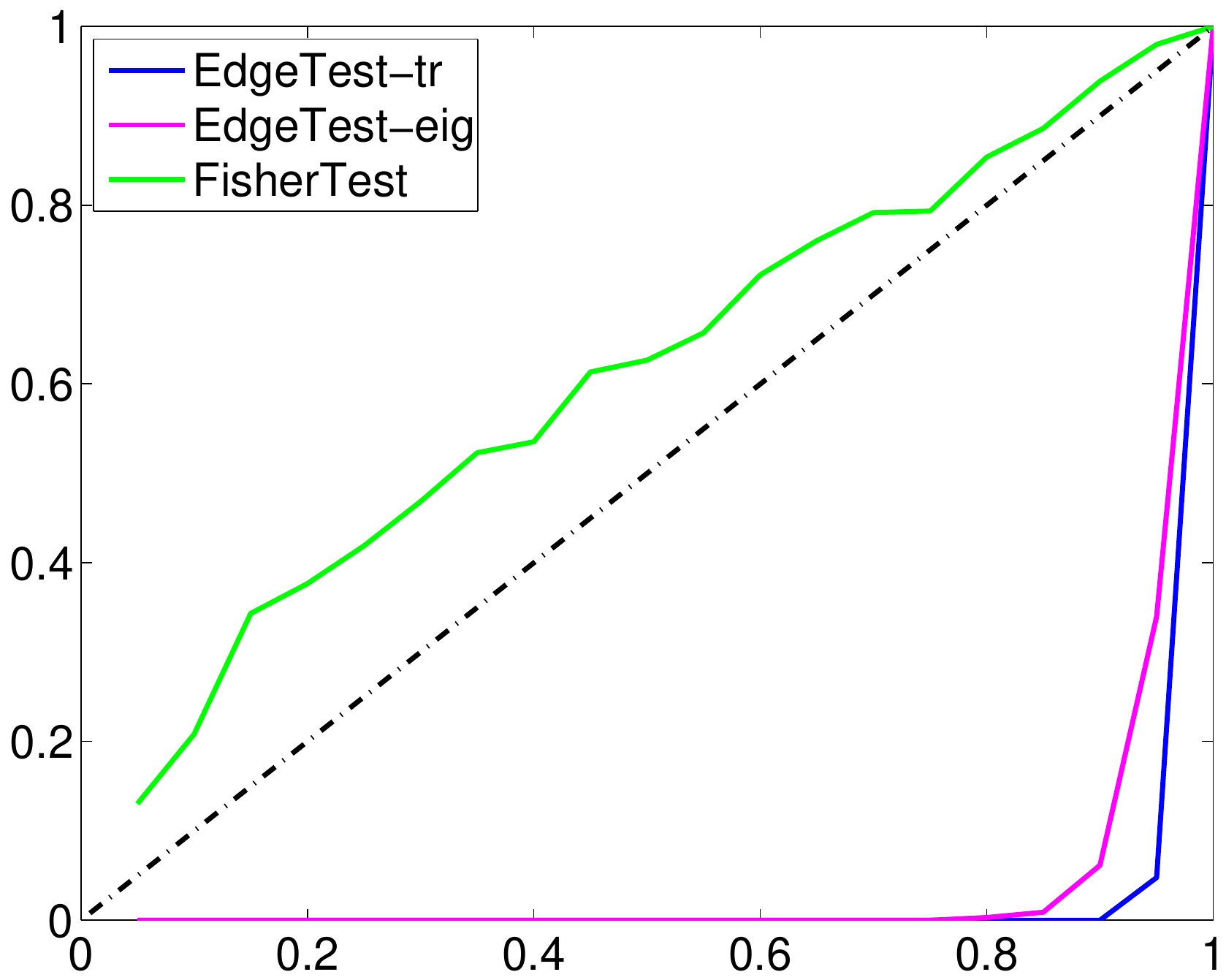} \\
& significance level $\delta$
\end{tabular}
\caption{Laplace dist., $n = 100 000$, $p=6$.}\label{EJS:fig:laplace_fpr_Normalized}
\end{subfigure}\par

\\

\begin{subfigure}{.45\columnwidth}
\setlength{\tabcolsep}{0.1em}
\renewcommand{\arraystretch}{0.5}
\begin{tabular}{cc}
&  \\
\begin{sideways} \qquad \qquad false positive rate \end{sideways} & \includegraphics[width=\linewidth]{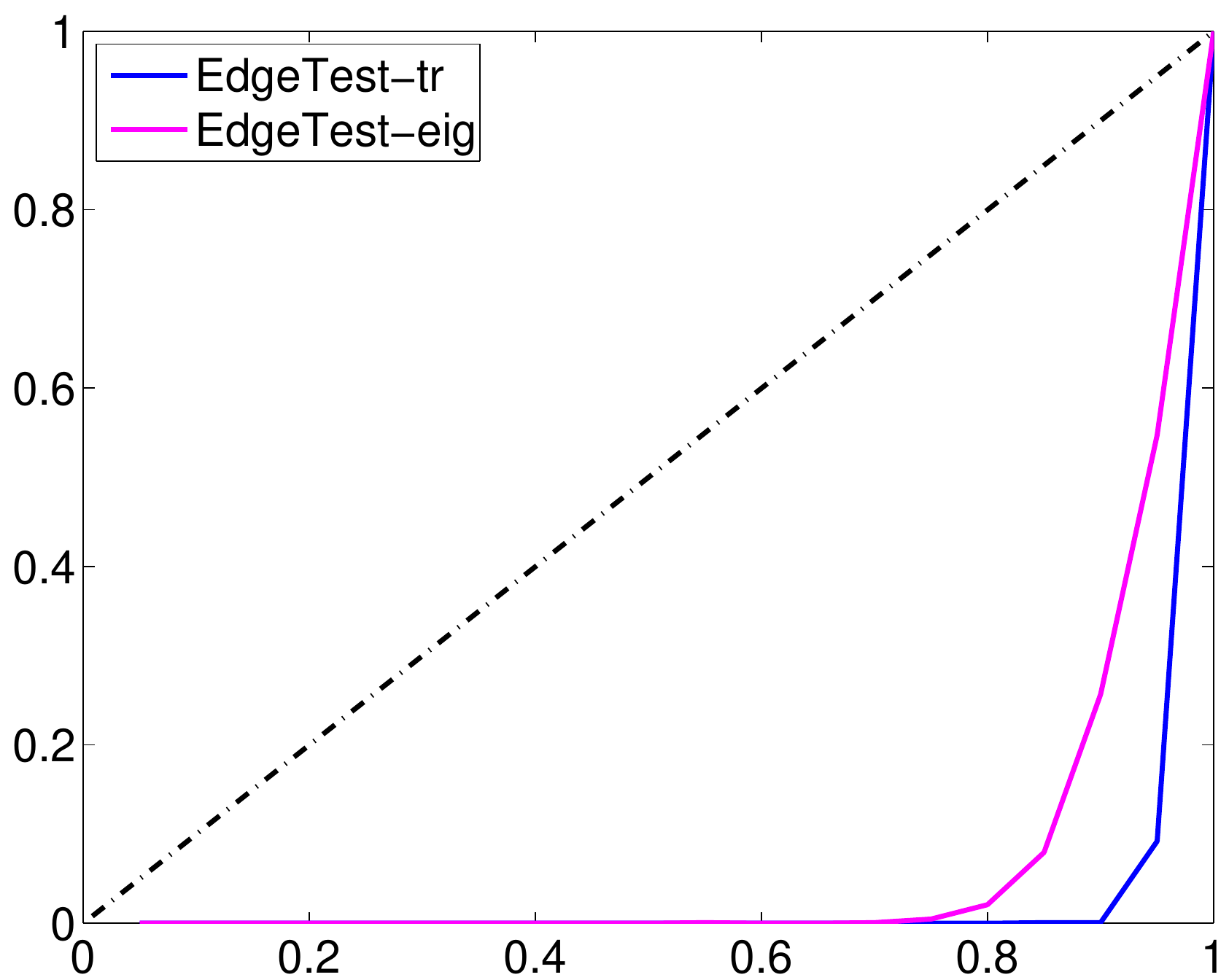} \\
& significance level $\delta$
\end{tabular}
\caption{Gaussian dist., $n= 100 000$, $p=6$.}\label{EJS:fig:gaussian_fpr_allentries_Normalized}
\end{subfigure}\hfill

&&

\begin{subfigure}{.45\textwidth}
\setlength{\tabcolsep}{0.1em}
\renewcommand{\arraystretch}{0.5}
\begin{tabular}{cc}
\begin{sideways} \qquad \qquad  false positive rate \end{sideways} & \includegraphics[width=\columnwidth]{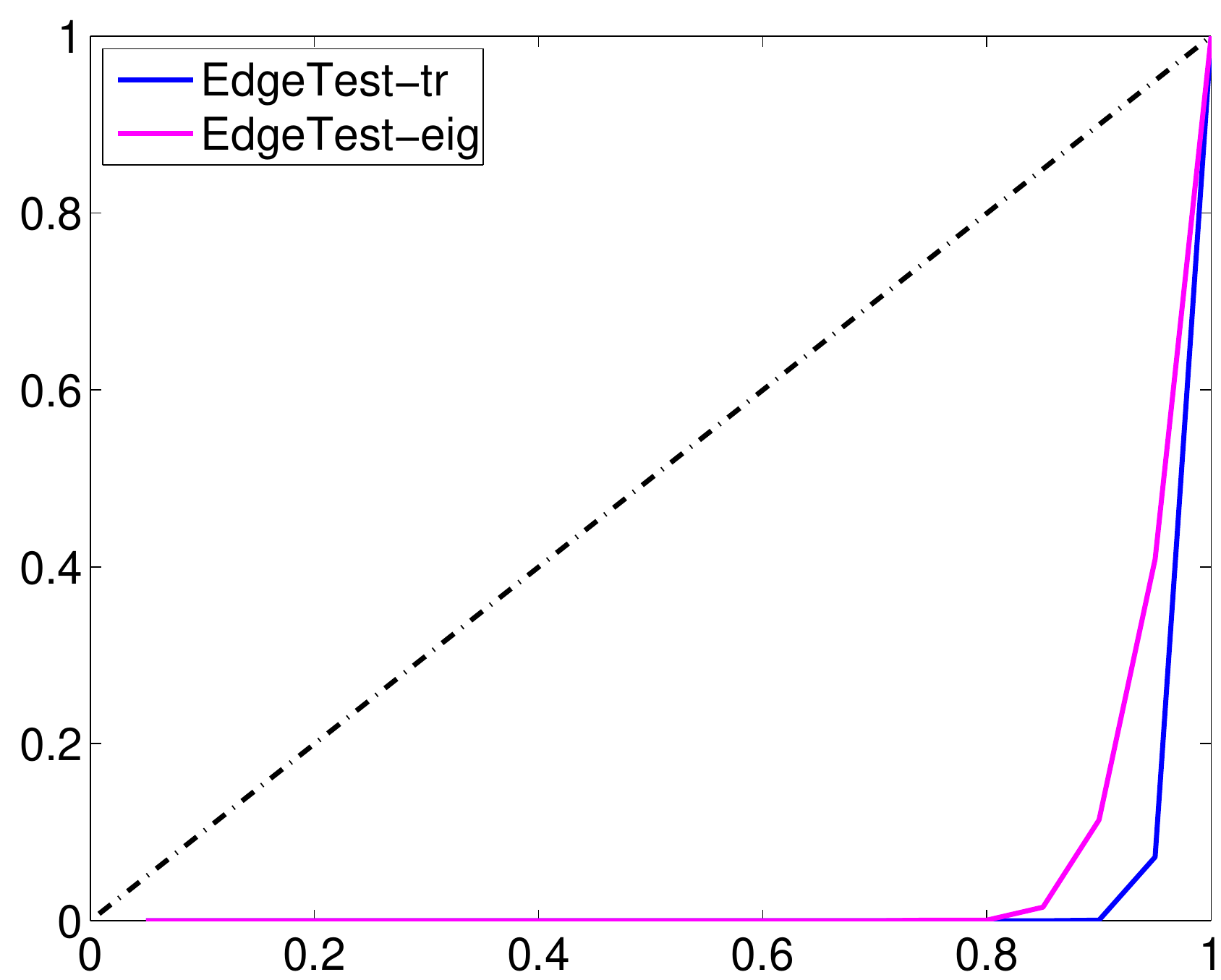} \\
& significance level $\delta$
\end{tabular}
\caption{Laplace dist., $n= 100 000$, $p=6$.}\label{EJS:fig:laplace_fpr_allentries_Normalized}
\end{subfigure}\par

\end{tabular}
\caption{We compare the false positive rate for the proposed test and the Fisher test. For the Gaussian distribution (Fig~\ref{EJS:fig:gaussian_fpr_Normalized}), the curves show that the Fisher test is well calibrated and that the proposed test is conservative (below the diagonal).  Furthermore, for the Laplace distribution (Fig~\ref{EJS:fig:laplace_fpr_Normalized}), the Fisher test does not obey the semantics of a bound on $\delta$ (the curve is above the diagonal) while by contrast, the proposed test remains conservative and sound. In Fig~\ref{EJS:fig:gaussian_fpr_allentries_Normalized} and Fig.~\ref{EJS:fig:laplace_fpr_allentries_Normalized}, we compare the rate of violating a bound on the true precision matrix as a function of $\delta$, i.e when $| \hat{\Theta}_{ij} - \Theta_{ij}| > t$ for an $(i,j)$ in $U(\Theta)$.}
\label{EJS:figure:different_distribution_fpr_witherrorbar}
\end{figure*}

\begin{figure*}
\centering
\begin{tabular}{ccc}

\begin{subfigure}{.45\columnwidth}
\setlength{\tabcolsep}{0.1em}
\renewcommand{\arraystretch}{0.5}
\begin{tabular}{cc}
\begin{sideways} \qquad \qquad \qquad  power \end{sideways} & \includegraphics[width=\linewidth]{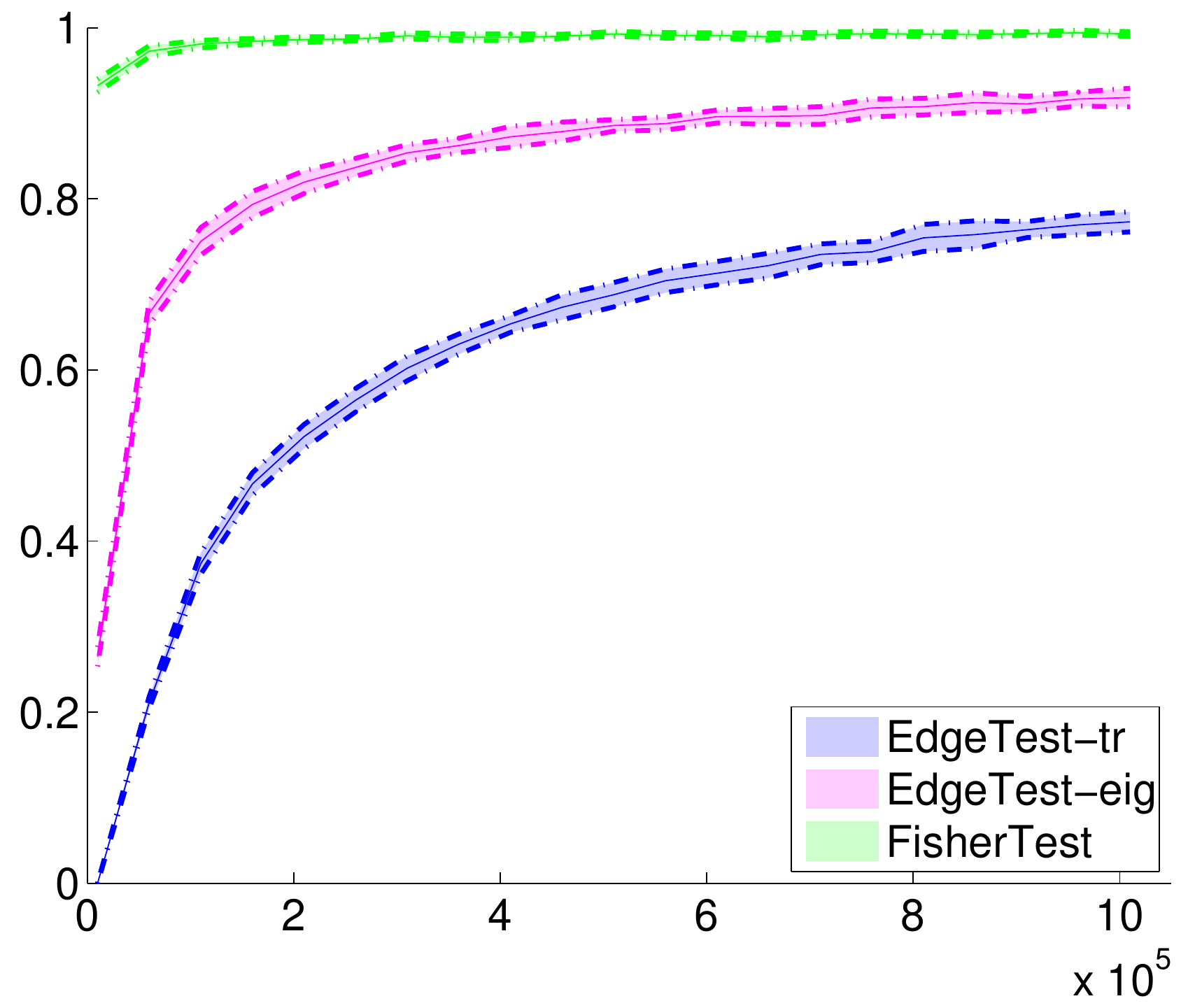} \\
& sample size $n$
\end{tabular}
\caption{Gaussian dist.}\label{EJS:fig:gaussian_power_Normalized}
\end{subfigure}\hfill

&&

\begin{subfigure}{.45\textwidth}
\setlength{\tabcolsep}{0.1em}
\renewcommand{\arraystretch}{0.5}
\begin{tabular}{cc}
\begin{sideways} \qquad \qquad \qquad power \end{sideways}
& \includegraphics[width=\columnwidth]{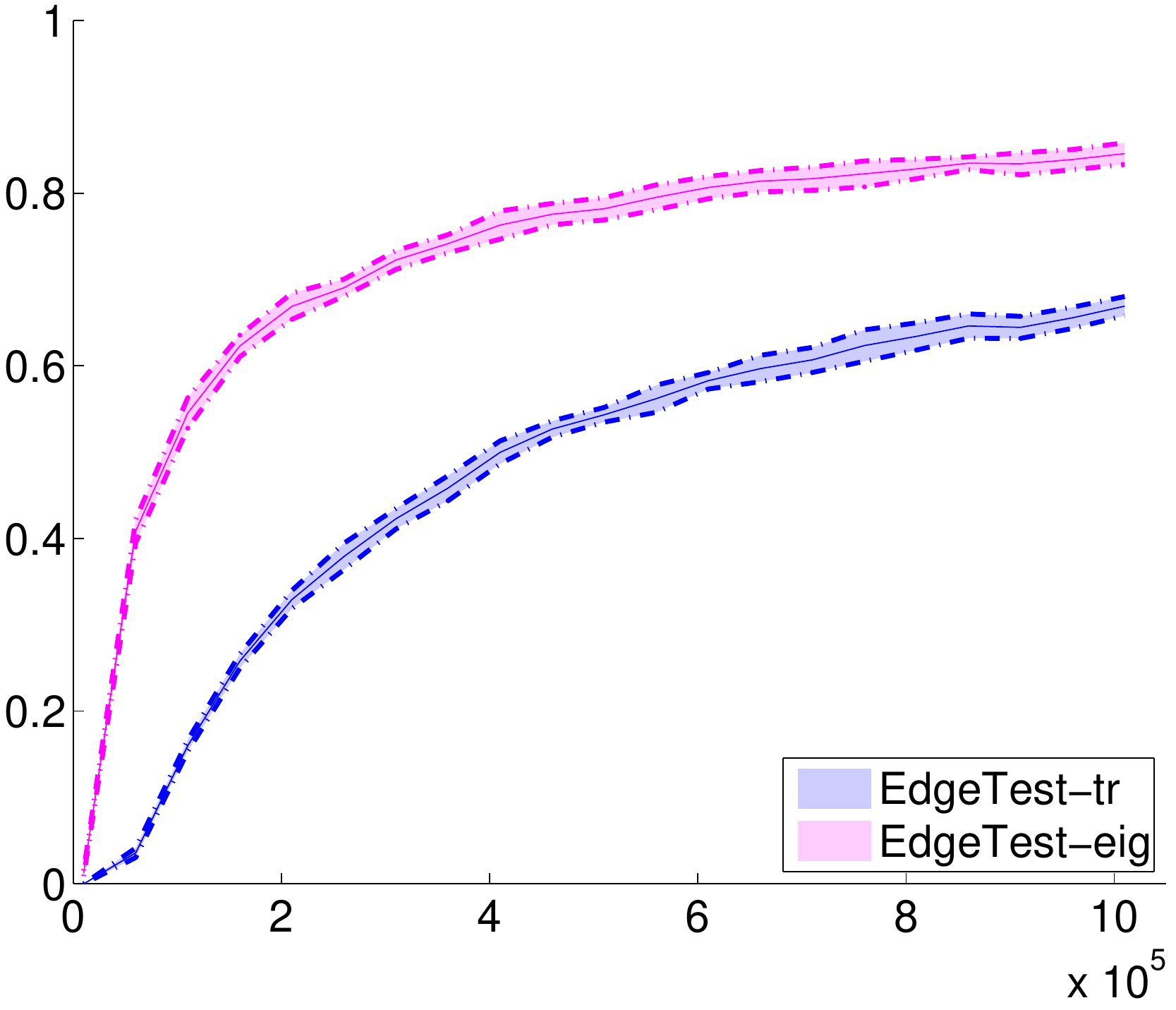} \\
& sample size $n$
\end{tabular}
\caption{Laplace dist.}\label{EJS:fig:laplace_power_Normalized}
\end{subfigure}\par

\\

\begin{subfigure}{.45\columnwidth}
\setlength{\tabcolsep}{0.1em}
\renewcommand{\arraystretch}{0.5}
\begin{tabular}{cc}
\begin{sideways} \qquad \qquad \qquad  power \end{sideways} & \includegraphics[width=\linewidth]{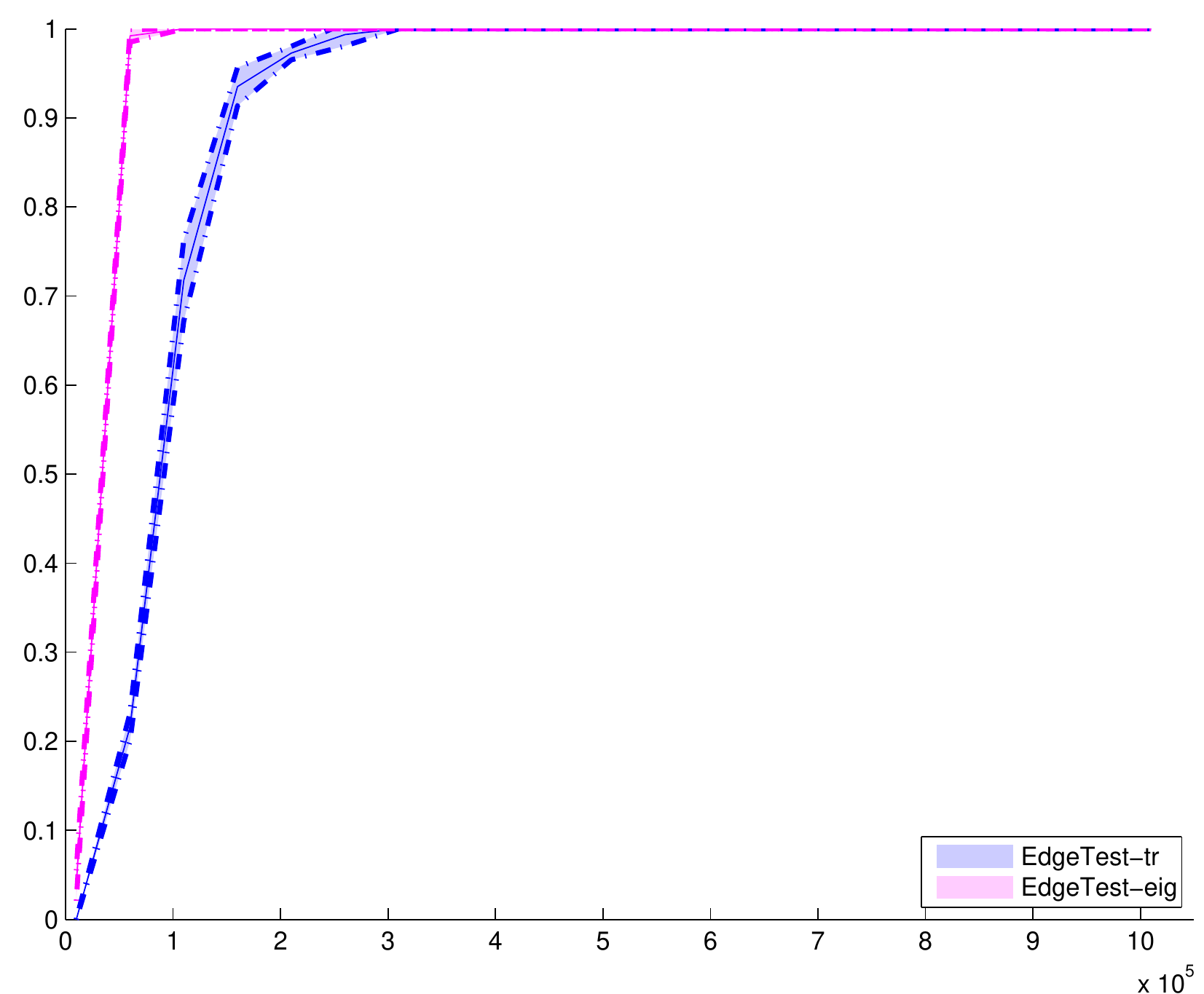} \\
& sample size $n$
\end{tabular}
\caption{Laplace dist.}\label{EJS:fig:gaussian_power_Normalized_seeEffect}
\end{subfigure}\hfill

&&

\begin{subfigure}{.45\textwidth}
\setlength{\tabcolsep}{0.1em}
\renewcommand{\arraystretch}{0.5}
\begin{tabular}{c}
\includegraphics[width=\columnwidth]{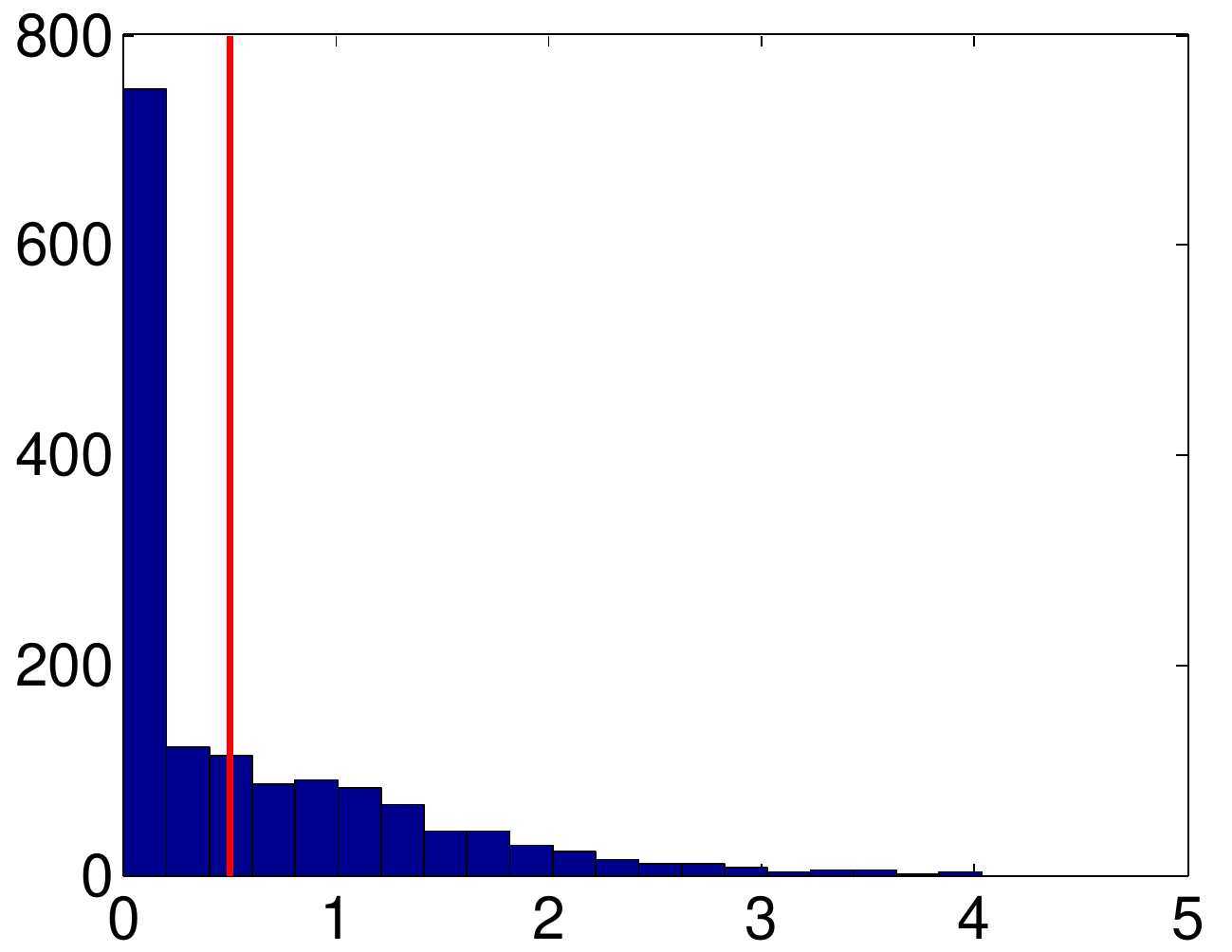} \\
\end{tabular}
\caption{Histogram of $|\Theta_{ij}|$ for 100 random graphs drawn from a Laplace distribution.}\label{EJS:fig:histogram_entries_Prec}
\end{subfigure}\par

\end{tabular}
\caption{As a function of the sample size $n$, we compare the power of the proposed test and the Fisher test for the Gaussian distribution (Fig.~\ref{EJS:fig:gaussian_power_Normalized}) and for the 
Laplace distribution (Fig.~\ref{EJS:fig:laplace_power_Normalized}). In Fig.~\ref{EJS:fig:gaussian_power_Normalized_seeEffect}, we plot the power of the proposed test when we reject the null hypothesis and when $|\Theta_{ij}| > 0.5$ (see histogram~\ref{EJS:fig:histogram_entries_Prec}). The shaded region indicates the standard error estimated from multiple repetitions.  The proposed tests are more generally applicable than the Fisher test, and have high power for edges with strong effects, i.e.\ those which are most important to detect and model.} \label{EJS:fig:samplesize_vs_power_test_for_differents_distributions}
\end{figure*}

%% file: Conclusion.tex
\section{Discussion}
We have considered the problem of structure discovery for undirected graphical models in the context of non-Gaussian multivariate distributions, use a concentration bound for $U$-statistics, leading to two probabilistic bounds $t_{\operatorname{Eig}}$ and $t_{\operatorname{Trace}}$. As a baseline, we compare to the Fisher test which is only correct under the assumption of a Gaussian distribution. As shown in the simulation studies, for non-Gaussian distributions, the Fisher test is not calibrated, while alternatively, the proposed test is conservative
Among the two probabilistic bounds presented here, the eigenvalue bound is preferred when availability of data is more limited than computation, while $t_{\operatorname{Trace}}$ is a competitive test when we have a fixed computational budget $N$.

\section{Conclusion}

In this work, we have constructed a conservative threshold on the absolute value of the precision matrix as a hypothesis test of the presence of an edge in a graphical model.  For a wider range of distributions, we have developed a threshold based on a $U$-statistic empirical estimator of the covariance matrix.  This is achieved by probabilistically bounding the distortion of the true covariance matrix, and then using this fixed bound in conjunction with Weyl's theorem to bound the distortion of the precision matrix.  These bounds are applicable to the quantification of uncertainty in the magnitude of an effect between variables as measured by the value of the precision matrix, and can also be used to construct a hypothesis test of whether an edge is present in a graphical model by testing for significant deviations from zero.  The resulting test asymptotically converges at the same $\frac{1}{\sqrt{n}}$ rate as the $U$-statistic, which we have additionally verified empirically.  We have shown two alternative thresholds, one based on the largest eigenvalue of $\Cov(\hat{\Sigma})$, and a second based on the trace of $\Cov(\hat{\Sigma})$, which strictly upper bounds the first.  Given arbitrary computation, we clearly favor the eigenvalue based approach, but for larger graphs with a large number of samples, the tighter threshold yields a test with computational complexity $\mathcal{O}(np^4)$ (due to the requirement of estimating $\mathcal{O}(p^4)$ entries of $\Cov(\hat{\Sigma})$ each of which has linear complexity) while the second has reduced complexity $\mathcal{O}(np^2 + p^3)$ as we need only compute the $\mathcal{O}(p^2)$ diagonal elements of $\Cov(\hat{\Sigma})$.  We have shown that this reduced complexity makes the trace bound competitive when computation rather than data availability is the restrictive factor.

The construction of the test threshold has upper bounded the $\| \cdot \|_2$ matrix norm with the Frobenius norm, which leads to a systematic overestimation of the threshold proportional to the size of the graph.  This is clearly demonstrated in the simulation study section.  We have taken the approach of probabilistically bounding the distortion of the covariance, and then, given this fixed bound, uniformly bounding the distortion of the precision matrix.  It may be of interest to consider a non-uniform bound to reduce the growth of the bound in the number of variables.

Simulation studies show that the test successfully recovers the structure of undirected graphical models given a sufficient number of samples.  The sample complexity increases with the size of the smallest non-zero entry of $\Theta$ as well as with the number of variables in the model.  Figure~\ref{EJS:fig:sample_vs_threshold} demonstrates that the bound tends to grow with the size of the graph for a fixed sample size, while the size of the non-zero entries follows the same distribution in these experiments.  Nevertheless, large values of $\Theta$ can be recovered with significance even in these cases.  The fact that the test was able to compute correct results even for $n=10^6$ and $p = 14$ in a short time demonstrates the scalability and soundness of the approach.

%% file: Appendix_Proofs_derivation_CovCov.tex
\section{Derivation of the covariance of the $U$-statistics for the covariance matrix} \label{EJS:sec:appendix_proof_covcov}

In this appendix, we show the details of the derivation of Theorem~\ref{EJS:theorem:covariance_covariance_ustatistic}. We derive low variance, unbiased estimates of the covariance between two $U$-statistics estimates $\hat{\Sigma}_{ij}$ and $\hat{\Sigma}_{kl}$, where $(i,j,k,l)$ range over each of the $d$ variates in a covariance matrix $\hat{\Sigma}$. We note $h$ and $g$ the corresponding kernel of order 2 for $\hat{\Sigma}_{ij}$ and  $\hat{\Sigma}_{kl}$, where
\begin{align}
h(u_1,u_2) &= \dfrac{1}{2} \left( X_{i_1} - X_{i_2} \right)  \left( X_{j_1} - X_{j_2} \right) \hspace{.05pt}, \mbox{with } u_r=(X_{i_r},X_{j_r})^T \\
g(v_1,v_2) &= \dfrac{1}{2} \left( X_{k_1} - X_{k_2} \right)  \left( X_{l_1} - X_{l_2} \right) \hspace{.05pt}, \mbox{with } v_r=(X_{k_r},X_{l_r})^T.
\end{align}
Then, the covariance $\Cov(\hat{\Sigma}_{ij},\hat{\Sigma}_{kl}) $ for the two $U$-statistics $\hat{\Sigma}_{ij}$ and $\hat{\Sigma}_{kl}$ is 
\begin{align} \label{EJS:appendix:eq:CovCov}
\Cov(\hat{\Sigma}_{ij},\hat{\Sigma}_{kl}) 
&= \binom{m}{2}^{-1} \left( 2 (m-2) \zeta_1 + \zeta_2 \right) \\
&= \binom{m}{2}^{-1} \left( 2 (m-2) \zeta_1  \right) + \mathcal{O}(m^{-2}) \nonumber
\end{align}
where $\zeta_1 = \Cov \left(  \E_{u_2}[h(u_1,u_2)], \E_{v_2}[g(v_1,v_2)] \right) $.

Depending on the equality and inequality of these four index variables, the empirical covariance estimate takes a different kernel form.  We have employed a computer assisted proof to determine that there are seven different forms and that each of the unique $\binom{p^2-\binom{p}{2}}{2}$ entries in $\operatorname{Cov}(\hat{\Sigma})$ (cf.\ Eq.~\eqref{EJS:eq:Covariance_of_Ustat_estimator}) can be mapped to one of these seven cases by a simple variable substitution.

In the sequel, we first describe the algorithm that determines the seven cases (Sec.~\ref{EJS:appendix:subsec:algorithm_forsevencases}), we derive empirical estimators for each of these seven cases (Sec.~\ref{EJS:appendix:subsec:derivation_sevencases}) and show that in all cases we have linear computation time in the number of samples (Sec.~\ref{EJS:appendix:subsec:proof_linear_computation_time}).

\subsection{Description of the algorithm providing the seven cases} \label{EJS:appendix:subsec:algorithm_forsevencases}

We formally described the algorithm that provided us 7 cases for the derivation of $\Cov(\hat{\Sigma}_{ij},\hat{\Sigma}_{kl})$ of Theorem~\ref{EJS:theorem:covariance_covariance_ustatistic}, where $\left(i,j,k,l \right)$ vary over the set of $d$ variables. 

\begin{description}
\item[Enumeration] 
First, we enumerate all configurations of $\Cov(\hat{\Sigma}_{ij},\hat{\Sigma}_{kl})$, which can be encoded as a non-unique assignment matrix of variables $i,j,k,l$ to instantiated variables $\left(a,b,c,d \right)$. For a fixed assignment of $i$ to variable $a$, we can list all possible assignments of the 3 remaining variables  $\left(j,k,l\right)$ to any $\left(a,b,c,d\right)$.  Na\"{i}vely, we have $4^3$ possible assignments, but many of them will be equivalent by variable substitution.  To test whether two forms are equivalent, it is sufficient to test a reduced form for equality.
\item[Reduced Form] We map a variable assignment to a reduced form by re-labeling variables sorted by the number of occurrences, which reduces the number of possible matches up-to non-uniqueness of the mapping due to equal numbers of variable occurrences.  This ambiguity is then resolved by testing for symmetries.  
\item[Symmetry] 
Symmetry of the covariance operator brings the following equally that we take into consideration in testing for equivalence:
\begin{align}
\Cov(\hat{\Sigma}_{ij},\hat{\Sigma}_{kl}) &= 
\Cov(\hat{\Sigma}_{kl},\hat{\Sigma}_{ij}) =       
\Cov(\hat{\Sigma}_{ij},\hat{\Sigma}_{lk}) =        
\Cov(\hat{\Sigma}_{lk},\hat{\Sigma}_{ij})  \\            
&= \Cov(\hat{\Sigma}_{lk},\hat{\Sigma}_{ji}) =        
\Cov(\hat{\Sigma}_{ji},\hat{\Sigma}_{kl}) =         
\Cov(\hat{\Sigma}_{ji},\hat{\Sigma}_{lk}) \nonumber.
\end{align}
\end{description}
The algorithm outputs each variable assignment that is not equivalent by variable substitution to any previously enumerated assignment. 
Open source code for the computer assisted proof is available at \projecturl.

The seven different cases are enumerated in Table~\ref{EJS:table:enumeration7cases}.

\begin{table}[h!]
\begin{tabular}{|r|l|l|}
\hline
Cases & Indices & Correspondence \\
\hline
1 & $i \ne j,k,l$; $j \ne k,l$; $k \ne l$  & $\Cov(\hat{\Sigma}_{ij},\hat{\Sigma}_{kl})$  \\
2 & $i=j$; $j \ne k,l$; $k = l$ &  $\Cov(\hat{\Sigma}_{ii},\hat{\Sigma}_{kk})$  \\
3 & $i=j$; $j \ne k,l$; $k \ne l$ & $\Cov(\hat{\Sigma}_{ii},\hat{\Sigma}_{kl})$  \\
4 & $i=k$; $j \ne i,k,l$; $k \ne l$ & $\Cov(\hat{\Sigma}_{ij},\hat{\Sigma}_{il})$  \\
5 & $i=k$; $i \ne j$; $j=l$; & $\Var(\hat{\Sigma}_{ij})$  \\
6 & $i=j=k$; $i \ne l$ & $\Cov(\hat{\Sigma}_{ii},\hat{\Sigma}_{il})$  \\
7 & $i=j,k,l$ & $ \Var(\hat{\Sigma}_{ii})$ \\
\hline
\end{tabular}
\caption{Enumeration and correspondence of the seven cases.} \label{EJS:table:enumeration7cases}
\end{table}

\subsection{The seven exhaustive cases} \label{EJS:appendix:subsec:derivation_sevencases}

We now derive linear-time finite-sample estimates of the covariance for each of the seven cases.

\paragraph{Notation}
\begin{itemize}
\item[-] $\Xbar{XYUV} = \E [XYUV]$
\item[-] $\Xbar{XYZ} = \E [XYZ]$
\item[-] $\Xbar{XY} = \E [XY]$
\item[-] $\Xbar{X} = \E [X]$
\item[-] $\Xbar{XYUV} \myspace \Xbar{X} = \E [XYUV] \times \E [X]$
\end{itemize}

\subsubsection{Case 1: $i \ne j,k,l$; $j \ne k,l$; $k \ne l$}

The kernels are 

\noindent\begin{minipage}{.5\linewidth}
\begin{align*}
&h(u_1,u_2) = \dfrac{1}{2} \left( X_{i_1} - X_{i_2} \right) \left( X_{j_1} - X_{j_2} \right); \\
&\E_{u_2}[h(u_1,u_2)] = \dfrac{1}{2} \left( X_{i_1} - \Xbar{X_i} \right) \left( X_{j_1} -\Xbar{X_j} \right);
\end{align*}
\end{minipage}\hfill
\begin{minipage}{.5\linewidth}
\begin{align*}
&g(v_1,v_2) = \dfrac{1}{2} \left( X_{k_1} - X_{k_2} \right)  \left( X_{l_1} - X_{l_2} \right) \\
&\E_{u_2}[g(v_1,v_2)] = \dfrac{1}{2} \left( X_{k_1} - \Xbar{X_k} \right)  \left( X_{l_1} - \Xbar{X_l} \right)
\end{align*}
\end{minipage}

\begin{align}
\zeta_1 &= \Cov \left[ 
\dfrac{1}{2} \left( X_{i_1} - \Xbar{X_i} \right) \left( X_{j_1} -\Xbar{X_j} \right), 
\dfrac{1}{2} \left( X_{k_1} - \Xbar{X_k} \right)  \left( X_{l_1} - \Xbar{X_l} \right)
\right] \\
&= \dfrac{1}{4} \biggl\{ \Cov \left[ 
X_{i_1} X_{j_1} - \Xbar{X_i} X_{j_1} - X_{i_1} \Xbar{X_j} ;
X_{k_1} X_{l_1} - \Xbar{X_k} X_{l_1} - X_{k_1}  \Xbar{X_l}
\right]
\biggr\} \nonumber \\
&= \dfrac{1}{4} \biggl\{ \E_{u_1} \big[ 
X_{i_1} X_{j_1} X_{k_1} X_{l_1} - \Xbar{X_i} X_{j_1} X_{k_1} X_{l_1} -  X_{i_1} \Xbar{X_j}  X_{k_1} X_{l_1} \nonumber \\
&\qquad\qquad\qquad - X_{i_1} X_{j_1} \Xbar{X_k} X_{l_1} + \Xbar{X_i} X_{j_1} \Xbar{X_k} X_{l_1} + X_{i_1} \Xbar{X_j} \myspace \Xbar{X_k}  X_{l_1} \nonumber \\
&\qquad\qquad\qquad - X_{i_1} X_{j_1} X_{k_1}  \Xbar{X_l} + \Xbar{X_i} X_{j_1} X_{k_1}  \Xbar{X_l} + X_{i_1} \Xbar{X_j} X_{k_1}  \Xbar{X_l} \big] \nonumber \\
& \qquad - \E_{u_1} \left[ X_{i_1} X_{j_1} - \Xbar{X_i} X_{j_1} - X_{i_1} \Xbar{X_j} \right] 
\E_{u_1} \left[ X_{k_1} X_{l_1} - \Xbar{X_k} X_{l_1} - X_{k_1}  \Xbar{X_l} \right]
\biggr\} \nonumber \\
&= \dfrac{1}{4} \biggl\{ 
\Xbar{X_{i} X_{j} X_{k} X_{l}} - \Xbar{X_i} \myspace \Xbar{X_{j} X_{k} X_{l}} -  \Xbar{X_j} \myspace  \Xbar{X_{i}  X_{k} X_{l}} \nonumber \\
&\qquad - \Xbar{X_k} \myspace  \Xbar{X_{i} X_{j}  X_{l}} + \Xbar{X_i} \myspace  \Xbar{X_k} \myspace  \Xbar{ X_{j}  X_{l}} + \Xbar{X_j} \myspace  \Xbar{X_k} \myspace  \Xbar{X_{i}  X_{l}} \nonumber \\
&\qquad - \Xbar{X_{i} X_{j} X_{k}} \myspace \Xbar{X_l} + \Xbar{X_i} \myspace \Xbar{X_l} \myspace \Xbar{ X_{j} X_{k}}   + \Xbar{X_j} \myspace \Xbar{X_l} \myspace \Xbar{X_{i}  X_{k}}    \nonumber \\
&\qquad - \left( \Xbar{X_i X_j} - 2  \myspace \Xbar{X_i} \myspace \Xbar{X_j} \right) \left( \Xbar{X_k X_l} - 2  \myspace\Xbar{X_k} \myspace \Xbar{X_l} \right)
\biggr\} \nonumber
\end{align}

\subsubsection{Case 2: $i=j$; $j \ne k,l$; $k = l$}

The kernels are 

\noindent\begin{minipage}{.5\linewidth}
\begin{align*}
&h(u_1,u_2) = \dfrac{1}{2} \left( X_{i_1} - X_{i_2} \right)^2;  \\
&\E_{u_2}[h(u_1,u_2)] = \dfrac{1}{2} \left( X_{i_1} - \Xbar{X_i} \right)^2;
\end{align*}
\end{minipage}\hfill
\begin{minipage}{.5\linewidth}
\begin{align*}
&g(v_1,v_2) = \dfrac{1}{2} \left( X_{k_1} - X_{k_2} \right)^2 \\
&\E_{u_2}[g(v_1,v_2)] = \dfrac{1}{2} \left( X_{k_1} - \Xbar{X_k} \right)^2
\end{align*}
\end{minipage}

Then, we have 
\begin{align}
\zeta_1 &= \Cov \left[ 
\dfrac{1}{2} \left( X_{i_1} - \Xbar{X_i} \right)^2;
\dfrac{1}{2} \left( X_{k_1} - \Xbar{X_k} \right)^2
\right] \\
&= \dfrac{1}{4} \biggl\{ \Cov \left[ 
X_{i_1}^2 - 2 X_{i_1} \Xbar{X_i} ;
X_{k_1}^2 - 2 X_{k_1} \Xbar{X_k}
\right] \biggr\} \nonumber \\
&= \dfrac{1}{4} \biggl\{ 
\E_{u_1} \left[ X_{i_1}^2 X_{k_1}^2 - 2 X_{i_1} \Xbar{X_i} X_{k_1}^2 - 2 X_{i_1}^2 X_{k_1} \Xbar{X_k} + 4 X_{i_1} \Xbar{X_i} X_{k_1} \Xbar{X_k} \right] \nonumber \\
&\qquad - \E_{u_1} \left[ X_{i_1}^2 - 2 X_{i_1} \Xbar{X_i}\right] \E_{u_1} \left[ X_{k_1}^2 - 2 X_{k_1} \Xbar{X_k} \right]
\biggr\} \nonumber \\
&= \dfrac{1}{4} \biggl\{ 
\Xbar{ X_{i}^2 X_{k}^2} - 2 \myspace \Xbar{X_i} \myspace \Xbar{X_{i} X_{k}^2}  - 2 \myspace \Xbar{X_{i}^2 X_{k_1}}  \myspace \Xbar{X_k} + 4 \Xbar{X_{i} X_{k}} \myspace \Xbar{X_i}  \myspace \Xbar{X_k} \nonumber \\
&\qquad -  \left( \Xbar{X_{i}^2} - 2 \myspace \Xbar{X_i}^2 \right) \left( \Xbar{X_{k}^2} - 2 \myspace \Xbar{X_k}^2 \right) 
\biggr\} \nonumber
\end{align}

\subsubsection{Case 3: $i=j$; $j \ne k,l$; $k \ne l$}

The kernels are 

\noindent\begin{minipage}{.5\linewidth}
\begin{align*}
&h(u_1,u_2) = \dfrac{1}{2} \left( X_{i_1} - X_{i_2} \right)^2 ;  \\
&\E_{u_2}[h(u_1,u_2)] = \dfrac{1}{2} \left( X_{i_1} - c \right)^2 ;
\end{align*}
\end{minipage}\hfill
\begin{minipage}{.5\linewidth}
\begin{align*}
&g(v_1,v_2) = \dfrac{1}{2} \left( X_{k_1} - X_{k_2} \right)  \left( X_{l_1} - X_{l_2} \right)  \\
&\E_{u_2}[g(v_1,v_2)] = \dfrac{1}{2} \left( X_{k_1} - \Xbar{X_k} \right)  \left( X_{l_1} - \Xbar{X_l} \right)
\end{align*}
\end{minipage}

Then, we have 
\begin{align}
\zeta_1 &= \Cov\left[ 
\dfrac{1}{2} \left( X_{i_1} - \Xbar{X_i} \right)^2 ; 
 \dfrac{1}{2} \left( X_{k_1} - \Xbar{X_k} \right)  \left( X_{l_1} - \Xbar{X_l} \right)
\right] \\
&= \dfrac{1}{4} \biggl\{ \Cov \left[
X_{i_1}^2 - 2 X_{i_1} \Xbar{X_i} ;
X_{k_1} X_{l_1} - \Xbar{X_k} X_{l_1} - X_{k_1} \Xbar{X_l}
\right] \biggr\} \nonumber \\
&= \dfrac{1}{4} \biggl\{ 
\E_{u_1} \big[  X_{i_1}^2 X_{k_1} X_{l_1} - 2 X_{i_1} \Xbar{X_i} X_{k_1} X_{l_1} - X_{i_1}^2 \Xbar{X_k} X_{l_1} \nonumber \\
&\qquad\qquad + 2 X_{i_1} \Xbar{X_i} \myspace \Xbar{X_k} X_{l_1} - X_{i_1}^2 X_{k_1} \Xbar{X_l} + 2 X_{i_1} \Xbar{X_i} X_{k_1} \Xbar{X_l} \big] \nonumber \\
&\qquad - \E_{u_1} \left[ X_{i_1}^2 - 2 X_{i_1} \Xbar{X_i}  \right] 
\E_{u_1} \left[ X_{k_1} X_{l_1} - \Xbar{X_k} X_{l_1} - X_{k_1} \Xbar{X_l} \right] 
\biggr\} \nonumber \\
&= \dfrac{1}{4} \biggl\{ 
\Xbar{ X_{i}^2 X_{k} X_{l} } - 2 \myspace \Xbar{X_{i} X_{k} X_{l} } \myspace \Xbar{X_i} - \Xbar{X_{i}^2 X_{l}} \myspace \Xbar{X_k}  \nonumber \\
&\qquad\qquad + 2 \myspace \Xbar{X_{i} X_{l}} \myspace \Xbar{X_i} \myspace \Xbar{X_k}  - \Xbar{X_{i}^2 X_{k_1}} \myspace \Xbar{X_l} + 2 \myspace \Xbar{X_{i} X_{k}} \myspace \Xbar{X_i} \myspace \Xbar{X_l} \nonumber \\
&\qquad - \left( \Xbar{X_{i}^2} - 2 \myspace \Xbar{X_i}^2 \right) \left( \Xbar{X_{k} X_{l} } - 2 \myspace \Xbar{X_k} \myspace \Xbar{X_l}\right)
\biggr\} \nonumber
\end{align}

\subsubsection{Case 4: $i=k$; $j \ne i,k,l$; $k \ne l$}

The kernels are 

\noindent\begin{minipage}{.5\linewidth}
\begin{align*}
&h(u_1,u_2) = \dfrac{1}{2} \left( X_{i_1} - X_{i_2} \right) \left( X_{j_1} - X_{j_2} \right)  ;  \\
&\E_{u_2}[h(u_1,u_2)] = \dfrac{1}{2} \left( X_{i_1} - \Xbar{X_i} \right) \left( X_{j_1} - \Xbar{X_j}\right) ;
\end{align*}
\end{minipage}\hfill
\begin{minipage}{.5\linewidth}
\begin{align*}
&g(v_1,v_2) = \dfrac{1}{2} \left( X_{i_1} - X_{i_2} \right)  \left( X_{l_1} - X_{l_2} \right) \\
&\E_{u_2}[g(v_1,v_2)] = \dfrac{1}{2} \left( X_{i_1} - \Xbar{X_i} \right)  \left( X_{l_1} - \Xbar{X_l} \right)
\end{align*}
\end{minipage}

Then, we have 
\begin{align}
\zeta_1 &= \Cov \left[ 
\dfrac{1}{2} \left( X_{i_1} - \Xbar{X_i} \right) \left( X_{j_1} - \Xbar{X_j}\right)  ;
\dfrac{1}{2} \left( X_{i_1} - \Xbar{X_i} \right)  \left( X_{l_1} - \Xbar{X_l} \right)
\right] \\
&= \dfrac{1}{4} \biggl\{ \Cov \left[ 
X_{i_1} X_{j_1} - \Xbar{X_i} X_{j_1} - X_{i_1} \Xbar{X_j}  ;
X_{i_1} X_{l_1} - \Xbar{X_i} X_{l_1} - X_{i_1} \Xbar{X_l} 
\right] \biggr\} \nonumber \\
&= \dfrac{1}{4} \biggl\{ \E_{u_1} \big[ 
X_{i_1}^2 X_{j_1} X_{l_1} - \Xbar{X_i} X_{j_1} X_{i_1} X_{l_1} - X_{i_1}^2 \Xbar{X_j} X_{l_1} \nonumber \\ 
&\qquad\qquad - X_{i_1} X_{j_1} \Xbar{X_i} X_{l_1} + \Xbar{X_i}^2 \myspace X_{j_1} X_{l_1} + X_{i_1} \Xbar{X_j} \myspace \Xbar{X_i} X_{l_1} \nonumber \\
&\qquad\qquad - X_{i_1}^2 X_{j_1} \Xbar{X_l}  + \Xbar{X_i} X_{j_1} X_{i_1} \Xbar{X_l}  + X_{i_1}^2 \Xbar{X_j}  \Xbar{X_l} 
\big] \nonumber \\
&\qquad -  \E_{u_1} \left[ X_{i_1} X_{j_1} - \Xbar{X_i} X_{j_1} - X_{i_1} \Xbar{X_j} \right]  
\E_{u_1} \left[ X_{i_1} X_{l_1} - \Xbar{X_i} X_{l_1} - X_{i_1} \Xbar{X_l} \right] 
\biggr\} \nonumber \\
&= \dfrac{1}{4} \biggl\{ 
\Xbar{X_{i_1}^2 X_{j_1} X_{l_1}} - \Xbar{X_i} \myspace \Xbar{X_{j_1} X_{i_1} X_{l_1}} - \Xbar{X_{i_1}^2 X_{l_1} } \myspace \Xbar{X_j} \nonumber \\ 
&\qquad\qquad - \Xbar{X_{i_1} X_{j_1} X_{l_1}} \myspace \Xbar{X_i}  + \Xbar{X_i}^2 \myspace \Xbar{X_{j_1} X_{l_1}}  + \Xbar{X_{i_1} X_{l_1}} \myspace \Xbar{X_j} \myspace \Xbar{X_i}  \nonumber \\
&\qquad\qquad - \Xbar{X_{i_1}^2 X_{j_1} } \myspace \Xbar{X_l}  + \Xbar{X_i} \myspace \Xbar{X_{j_1} X_{i_1}} \myspace \Xbar{X_l}  + \Xbar{X_{i_1}^2} \myspace \Xbar{X_j} \myspace  \Xbar{X_l} 
\big] \nonumber \\
&\qquad -  \left( \Xbar{X_{i} X_{j}} - 2 \myspace \Xbar{X_i} \myspace \Xbar{X_j} \right)
\left(  \Xbar{X_{i} X_{l}} - 2 \myspace \Xbar{X_i} \myspace \Xbar{X_l} \right) 
\biggr\} \nonumber
\end{align}

\subsubsection{Case 5: $i=k$; $i \ne j$; $j=l$; }

\noindent\begin{minipage}{.5\linewidth}
\begin{align*}
&h(u_1,u_2) = \dfrac{1}{2} \left( X_{i_1} - X_{i_2} \right) \left( X_{j_1} - X_{j_2} \right);  \\
&\E_{u_2}[h(u_1,u_2)] = \dfrac{1}{2} \left( X_{i_1} - \Xbar{X_i} \right) \left( X_{j_1} - \Xbar{X_j} \right);
\end{align*}
\end{minipage}\hfill
\begin{minipage}{.5\linewidth}
\begin{align*}
&g(v_1,v_2) = h(u_1,u_2) \\
&\E_{u_2}[g(v_1,v_2)] = \E_{u_2}[h(u_1,u_2)]
\end{align*}
\end{minipage}

Then, we have 
\begin{align}
\zeta_1 &= \Var \left[ \dfrac{1}{2} \left( X_{i_1} - \Xbar{X_i} \right) \left( X_{j_1} - \Xbar{X_j} \right) \right]  \\
&= \dfrac{1}{4} \biggl\{ \Var \left[ X_{i_1} X_{j_1} - \Xbar{X_i} X_{j_1} - X_{i_1} \Xbar{X_j}  \right]  \biggr\} \nonumber \\
&= \dfrac{1}{4} \biggl\{ \E_{u_1} \left[ (X_{i_1} X_{j_1} - \Xbar{X_i} X_{j_1} - X_{i_1}  \Xbar{X_j} )^2 \right]- \E_{u_1} \left[ X_{i_1} X_{j_1} - \Xbar{X_i} X_{j_1} - X_{i_1} \Xbar{X_j}  \right]^2   \biggr\}  \nonumber \\
&= \dfrac{1}{4} \biggl\{ \E_{u_1} \big[ X_{i_1}^2 X_{j_1}^2 - 2 X_{i_1} X_{j_1}^2 \Xbar{X_i} + \Xbar{X_i}^2 X_{j_1}^2 
 - 2 X_{i_1}^2 X_{j_1}  \Xbar{X_j} + 2 \Xbar{X_i} X_{j_1} X_{i_1}  \Xbar{X_j} + X_{i_1}^2  \Xbar{X_j}^2 \big]  \nonumber \\
&\qquad \qquad  - \left( \Xbar{X_i X_j} - 2 (\Xbar{X_i} \myspace \Xbar{X_j}) \right)^2   \biggr\}  \nonumber \\
&= \dfrac{1}{4} \biggl\{  \Xbar{X_{i}^2 X_{j}^2} - 2 \Xbar{X_{i} X_{j}^2} \myspace \Xbar{X_i} + \Xbar{X_i}^2 \myspace\Xbar{X_{j}^2} 
 - 2 \Xbar{X_{i}^2 X_{j}} \myspace  \Xbar{X_j} + 2 \Xbar{X_i} \myspace \Xbar{X_j}  \myspace\Xbar{X_{j} X_{i}} + \Xbar{X_{i}^2} \myspace  \Xbar{X_j}^2 \nonumber \\
&\qquad \qquad - \left( \Xbar{X_i X_j} - 2 (\Xbar{X_i} \myspace \Xbar{X_j}) \right)^2   \biggr\}  \nonumber
\end{align}

\subsubsection{Case 6: $i=j=k$; $i \ne l$}

The kernels are 

\noindent\begin{minipage}{.5\linewidth}
\begin{align*}
&h(u_1,u_2) = \dfrac{1}{2} \left( X_{i_1} - X_{i_2} \right)^2 ;  \\
&\E_{u_2}[h(u_1,u_2)] = \dfrac{1}{2} \left( X_{i_1} - \Xbar{X_i} \right)^2 ;
\end{align*}
\end{minipage}\hfill
\begin{minipage}{.5\linewidth}
\begin{align*}
&g(v_1,v_2) = \dfrac{1}{2} \left( X_{i_1} - X_{i_2} \right)  \left( X_{l_1} - X_{l_2} \right) \\
&\E_{u_2}[g(v_1,v_2)] = \dfrac{1}{2} \left( X_{i_1} -\Xbar{X_i} \right)  \left( X_{l_1} - \Xbar{X_l} \right)
\end{align*}
\end{minipage}

Then, we have 
\begin{align}
\zeta_1 &= \Cov \left[
\dfrac{1}{2} \left( X_{i_1} - \Xbar{X_i} \right)^2 ;
\dfrac{1}{2} \left( X_{i_1} -\Xbar{X_i} \right)  \left( X_{l_1} - \Xbar{X_l} \right)
\right]  \\
&= \dfrac{1}{4} \biggl\{  \Cov \left[
X_{i_1}^2 - 2 X_{i_1} \Xbar{X_i};
X_{i_1} X_{l_1} - \Xbar{X_i} X_{l_1} - X_{i_1} \Xbar{X_l}
\right] \biggr\}  \nonumber \\
&= \dfrac{1}{4} \biggl\{  \E_{u_1} \big[
X_{i_1}^2 X_{i_1} X_{l_1} - 2 X_{i_1} \Xbar{X_i} X_{i_1} X_{l_1} - X_{i_1}^2 \Xbar{X_i} X_{l_1} \nonumber \\
&\qquad \qquad + 2 X_{i_1} \Xbar{X_i} \myspace \Xbar{X_i} X_{l_1} - X_{i_1}^2 X_{i_1} \Xbar{X_l} + 2 X_{i_1} \Xbar{X_i} X_{i_1} \Xbar{X_l} \big] \nonumber \\
&\qquad - \E_{u_1} \left[ X_{i_1}^2 - 2X_{i_1} \Xbar{X_i} \right] 
\E_{u_1} \left[ X_{i_1} X_{l_1} - \Xbar{X_i} X_{l_1} - X_{i_1} \Xbar{X_l} \right]  
\biggr\}  \nonumber \\
&= \dfrac{1}{4} \biggl\{  
\Xbar{X_{i}^3 X_{l}} - 3 \myspace \Xbar{X_{i}^2 X_{l}} \myspace \Xbar{X_i}  + 2 \myspace \Xbar{ X_{i} X_{l}} \myspace \Xbar{X_i}^2 - \Xbar{X_{i}^3} \myspace \Xbar{X_l} + 2 \myspace \Xbar{X_{i}^2} \myspace \Xbar{X_i} \myspace \Xbar{X_l}  \nonumber \\
&\qquad - \left( \Xbar{X_i^2} - 2 \myspace  \Xbar{X_i}^2 \right)
\left( \Xbar{X_i X_l} - 2 \myspace  \Xbar{X_i} \myspace \Xbar{X_l} \right)
\biggr\}  \nonumber
\end{align}

\subsubsection{Case 7: $i=j,k,l$}

The kernels are 

\noindent\begin{minipage}{.5\linewidth}
\begin{align*}
&h(u_1,u_2) = \dfrac{1}{2} \left( X_{i_1} - X_{i_2} \right)^2 ;  \\
&\E_{u_2}[h(u_1,u_2)] = \dfrac{1}{2} \left( X_{i_1} - \Xbar{X_i} \right)^2 ;
\end{align*}
\end{minipage}\hfill
\begin{minipage}{.5\linewidth}
\begin{align*}
&g(v_1,v_2) = h(u_1,u_2) \\
&\E_{u_2}[g(v_1,v_2)] =\E_{u_2}[h(u_1,u_2)]
\end{align*}
\end{minipage}

Then, we have 
\begin{align}
\zeta_1 &= \Var \left[ \dfrac{1}{2} \left( X_{i_1} - \Xbar{X_i} \right)^2 \right]  \\
&= \dfrac{1}{4} \Var \left[  X_{i_1}^2 - 2 X_{i_1} \Xbar{X_i} \right] \nonumber \\
&= \dfrac{1}{4} \biggl\{
\E_{u_1} \left[ \left( X_{i_1}^2 - 2 X_{i_1} \Xbar{X_i} \right)^2\right]
- \E_{u_1} \left[ X_{i_1}^2 - 2 X_{i_1} \Xbar{X_i} \right]^2
\biggr\}\nonumber \\
&= \dfrac{1}{4} \biggl\{
\Xbar{X_{i}^4} - 4 \Xbar{X_{i}^3} \myspace \Xbar{X_i} + 4 \Xbar{X_{i}^2} \myspace \Xbar{X_i}^2 
- \left( \Xbar{X_i^2} - 2 \Xbar{X_i}^2 \right)^2
\biggr\}\nonumber
\end{align}

\subsection{Derivation in $\mathcal{O}(n)$ time for all terms} \label{EJS:appendix:subsec:proof_linear_computation_time}

In section~\ref{EJS:appendix:subsec:derivation_sevencases}, all terms are in the form of $\E[X]$,$\E[XY]$,$\E[XYZ]$ and $\E[XYUV]$ and can be computed in $\mathcal{O}(n)$ as following

\begin{align}
\E[X] &= \frac{1}{m} \sum_{q=1}^n X_q \\
\E[XY] &= \frac{1}{m} \sum_{q=1}^n X_q \odot Y_q \\
\E[XYZ] &= \frac{1}{m} \sum_{q=1}^n X_q \odot Y_q \odot Z_q \\
\E[XYUV] &= \frac{1}{m} \sum_{q=1}^n X_q \odot Y_q \odot U_q \odot V_q
\end{align}